\documentclass{article}

\PassOptionsToPackage{numbers}{natbib}

\usepackage[preprint]{neurips_2025}

\usepackage[utf8]{inputenc} 
\usepackage[T1]{fontenc}    
\usepackage{hyperref}       
\usepackage{url}            
\usepackage{booktabs}       
\usepackage{amsfonts}       
\usepackage{nicefrac}       
\usepackage{microtype}      
\usepackage{xcolor}         
\usepackage{graphicx} 
\usepackage{algorithm}
\usepackage{algpseudocode}
\usepackage{amsmath, amsthm}
\usepackage{amsfonts}
\usepackage{amssymb}
\usepackage{booktabs}
\usepackage{multirow}
\usepackage{graphicx}
\usepackage{adjustbox}
\usepackage{svg}
\usepackage{wrapfig}
\usepackage{subcaption}
\usepackage{floatflt}
\usepackage{enumitem}
\usepackage{placeins}
\usepackage{titlesec}

\theoremstyle{definition}    
\newtheorem{theorem}{Theorem}
\algnewcommand{\CommonCall}[2]{\textsc{#1}\textnormal{\small(}#2\textnormal{\small)}}
\theoremstyle{definition}  

\newtheorem{lemma}[theorem]{Lemma}    
\newtheorem{corollary}{Corollary}[theorem] 

\title{Synergizing Reinforcement Learning and Genetic Algorithms for Neural Combinatorial Optimization}

\author{
  Shengda Gu\textsuperscript{1,2} \quad 
  Kai Li\textsuperscript{1,2}\thanks{Corresponding author.} \quad 
  Junliang Xing\textsuperscript{4} \quad 
  Yifan Zhang\textsuperscript{1,3} \quad 
  Jian Cheng\textsuperscript{1,3} \\
  \textsuperscript{1}C\textsuperscript{2}DL, Institute of Automation, Chinese Academy of Sciences \\
   \textsuperscript{2}School of Artificial Intelligence, University of Chinese Academy of Sciences\\
  \textsuperscript{3}School of Advanced Interdisciplinary Sciences, University of Chinese Academy of Sciences \\
  \textsuperscript{4}Department of Computer Science and Technology, Tsinghua University \\
}

\begin{document}

\maketitle

\begin{abstract}
Combinatorial optimization problems are notoriously challenging due to their discrete structure and exponentially large solution space. Recent advances in deep reinforcement learning (DRL) have enabled the learning heuristics directly from data. However, DRL methods often suffer from limited exploration and susceptibility to local optima. On the other hand, evolutionary algorithms such as Genetic Algorithms (GAs) exhibit strong global exploration capabilities but are typically sample inefficient and computationally intensive. In this work, we propose the Evolutionary Augmentation Mechanism (EAM), a general and plug-and-play framework that synergizes the learning efficiency of DRL with the global search power of GAs. EAM operates by generating solutions from a learned policy and refining them through domain-specific genetic operations such as crossover and mutation. These evolved solutions are then selectively reinjected into the policy training loop, thereby enhancing exploration and accelerating convergence. We further provide a theoretical analysis that establishes an upper bound on the KL divergence between the evolved solution distribution and the policy distribution, ensuring stable and effective policy updates. EAM is model-agnostic and can be seamlessly integrated with state-of-the-art DRL solvers such as the Attention Model, POMO, and SymNCO. Extensive results on benchmark problems—including TSP, CVRP, PCTSP, and OP—demonstrate that EAM significantly improves both solution quality and training efficiency over competitive baselines.
\end{abstract}

\section{Introduction}
\label{sec:intro}
Combinatorial Optimization Problems (COPs) are fundamental decision-making and optimization tasks in discrete spaces, with wide-ranging applications in domains such as logistics scheduling~\citep{xu2024optimization, ping2022application}, chip design~\citep{deliparaschos2018design, sha2024time}, and financial investment~\citep{luu2022application,fekri2019designing}. Due to their NP-hard nature, finding optimal solutions is often computationally expensive, posing significant challenges for real-time applications. As a result, efficiently obtaining high-quality approximate solutions has become a central goal in both academic research and engineering practice. Traditional heuristic approaches~\citep{helsgaun2017extension, vidal2022hybrid, lourencco2018iterated} typically rely on substantial expert knowledge, making them inflexible and difficult to generalize across different problems. These limitations have led to increasing interest in automated, data-driven solution paradigms that aim to reduce manual design effort and improve scalability.

In recent years, DRL has emerged as a promising framework for Neural Combinatorial Optimization (NCO), enabling policies to learn solution construction strategies in an end-to-end manner. DRL-NCO methods can be broadly categorized into Learning-to-Improve (L2I)~\citep{ahn2020learning,d2020learning,wu2021learning,chen2019learning,ma2021learning,kim2021learning} and Learning-to-Construct (L2C)~\citep{wu2024neural, bello2016neural, grinsztajn2023winner, hottung2024polynet, kim2022sym, kool2018attention, kwon2020pomo, sun2024learning, wang2024leader} approaches. Among them, L2C methods have been extensively studied due to strong empirical performance. However, L2C methods suffer from notable limitations: their autoregressive generation mechanism restricts the ability to alter or refine previously constructed partial solutions, often leading to suboptimal solutions. Additionally, the sparse reward structure typical of many COPs severely hampers efficient policy training.

To address these challenges—particularly the limited exploration capacity of standard policy learning—it is essential to introduce new learning mechanisms. Evolutionary Algorithms, especially Genetic Algorithms (GAs), are also used for solving COPs~\citep{uray2023csrx, mahmoudinazlou2024hybrid, alkafaween2024efficiency, vidal2022hybrid}. GAs offer strong global exploration capabilities via population-based search and genetic operators such as crossover and mutation. However, their lack of gradient-based guidance often results in sample inefficiency and high computational overhead when used in isolation. This contrast reveals the potential for a complementary integration: instead of replacing RL, GA can act as an auxiliary component to improve exploration and enhance sample quality within a unified framework.

Motivated by these observations, we introduce the Evolutionary Augmentation Mechanism (EAM), a hybrid framework that integrates RL with GA to jointly improve training efficiency and solution quality (Fig.~\ref{fig:overall_framework}). EAM first generates an initial solution set by sampling from the RL policy and then applies GA to iteratively evolve them, yielding solutions that are structurally diverse and of higher quality. These evolved solutions are subsequently combined with policy-generated samples to update the policy. This closed-loop process—wherein the policy accelerates GA and evolved outputs progressively refine the RL policy—fosters mutual reinforcement. Specifically, the policy provides well-structured initial solutions that serve as strong baselines for GA, improving its efficiency. GA, in contrast, yields local structural optimizations that are often beyond the reach of the autoregressive policy, thereby enriching the training data with diverse and exploratory feedback samples.

While incorporating GA enhances sample quality, it may introduce distributional biases to on-policy learning, potentially compromising gradient estimation. To quantify this effect, we theoretically model the distributional divergence between evolved solutions and policy samples, using the KL divergence to analyze its impact on policy gradient estimation. This theoretical foundation underpins the stability analysis of EAM. Building upon this insight, we propose a task-aware evolutionary hyperparameter selection strategy, which adapts selection, crossover, and mutation rates based on structural characteristics of specific COPs. By balancing perturbation intensity and distributional alignment, this strategy enhances learning stability and ensures robust generalization across tasks.

\begin{figure}[t]
    \centering
    \includegraphics[width=0.95\textwidth]{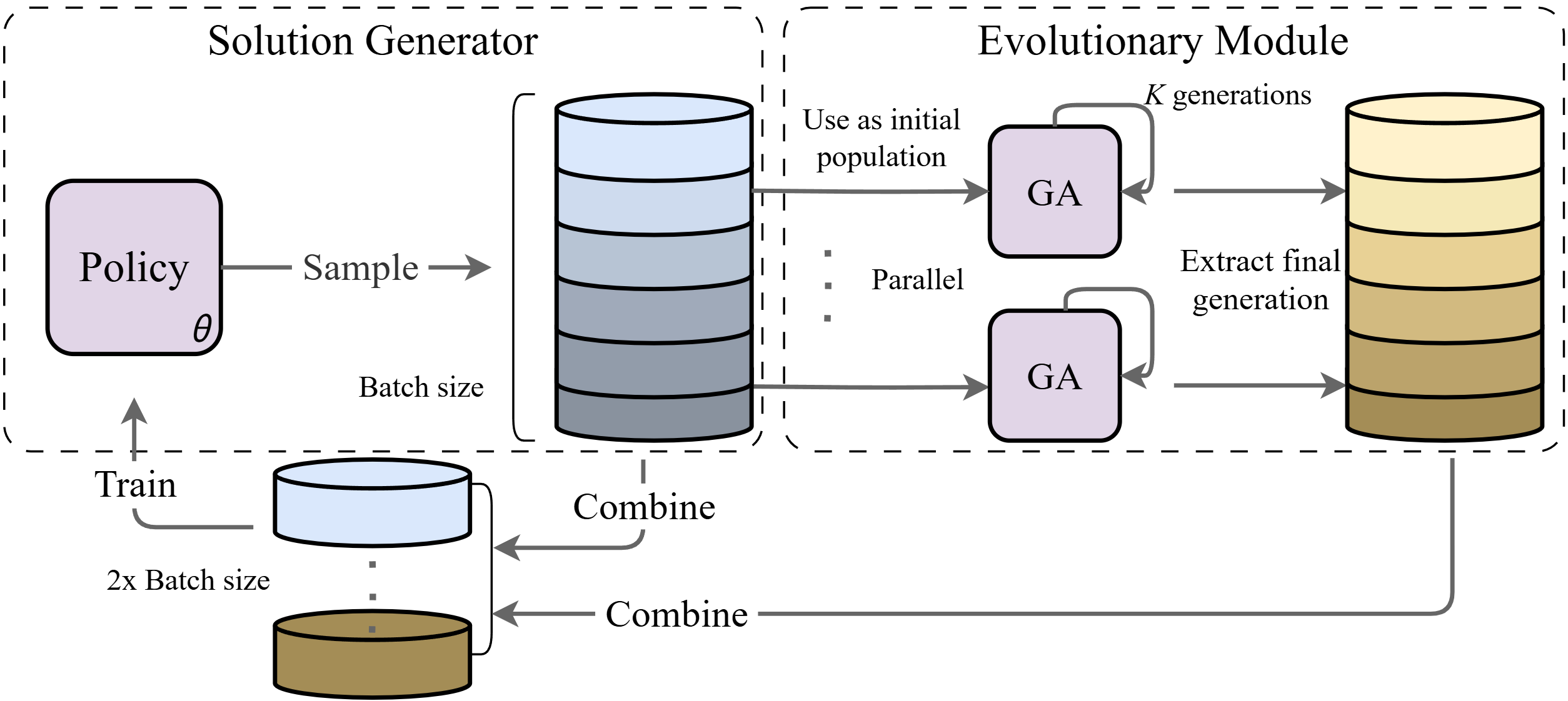}
    \caption{\textbf{An overview of the proposed Evolutionary Augmentation Mechanism (EAM).} Policy-sampled solutions are used to initialize the population of the Genetic Algorithm. The evolved solutions are then merged with the original samples and jointly used to train the policy network, forming a closed-loop learning and evolutionary framework.}
    \label{fig:overall_framework}
\end{figure}

We conduct extensive evaluations of EAM on four representative COPs: the Traveling Salesman Problem (TSP), Capacitated Vehicle Routing Problem (CVRP), Prize-Collecting TSP (PCTSP)~\citep{balas1989prize}, and Orienteering Problem (OP)~\citep{golden1987orienteering}. Experimental results demonstrate that EAM consistently improves the performance of various DRL-NCO solvers across diverse task structures and scales. EAM enhances solution quality while also facilitating faster policy convergence. When integrated into solvers such as AM, POMO, and SymNCO, EAM achieves faster convergence on CVRP tasks of varying scales, demonstrating its practicality as a model-agnostic policy enhancement module.

\section{Preliminaries}
\label{sec:preliminaries}
Solutions to COPs such as TSP, CVRP, PCTSP, and OP can be naturally represented as sequences of node visits. Formally, we define a solution (trajectory) \( \boldsymbol{\boldsymbol{\tau}} = [\tau_1, \tau_2, \ldots, \tau_n] \), where each \( \tau_i\) denotes the index of a node visited at step \( i \). The structure of \( \boldsymbol{\boldsymbol{\tau}} \) varies depending on the specific problem: in TSP, all nodes must be visited exactly once; in CVRP, the solution must visit each customer exactly once, and ensure that the total demand between depot returns does not exceed vehicle capacity; in PCTSP~\citep{balas1989prize} and OP~\citep{golden1987orienteering}, \textbf{only a subset of nodes} is visited. In PCTSP, the goal is to collect at least a minimum required prize while minimizing the travel cost, whereas in OP, the objective is to maximize the collected prize under a travel budget constraint. The reward function $R(\boldsymbol{\tau})$ is defined in accordance with the optimization objective, where higher rewards correspond to better solutions, such as shorter total distances or higher collected prizes. 

\section{Method}
EAM consists of two tightly integrated components: (1) a \textit{solution generator} that constructs feasible solutions via RL and (2) an \textit{evolutionary module} that refines these solutions using task-specific genetic operators. We next describe the two main components of our framework and present the theoretical foundation for setting the evolutionary hyperparameters within the augmentation mechanism.

\subsection{Solution Generators}
Under the L2C paradigm, solving COPs is commonly modeled as a sequential decision-making process~\citep{vinyals2015pointer}, where a parameterized policy incrementally constructs a feasible solution. Given an input instance $s$, the solution distribution is defined in an autoregressive form: $p(\boldsymbol{\boldsymbol{\tau}}| s) = \prod_{t=1}^T p(\tau_t|\boldsymbol{\tau}_{1:t-1}, s)$. We adopt a Transformer-based encoder-decoder architecture consistent with AM~\citep{kool2018attention} and POMO~\citep{kwon2020pomo}.

While this modeling strategy provides strong expressiveness and flexibility, it also imposes structural limitations: the autoregressive mechanism restricts the ability to revise early decisions. After a node is selected, it is permanently incorporated into the constructed solution. Consequently, the model cannot revise earlier decisions even when subsequent information reveals them to be suboptimal. This irrevocability leads to an accumulation of early-stage errors that degrade overall solution quality. Moreover, sparse rewards, which are only revealed upon completion of a full solution, make it difficult for the model to assign credit to individual actions, thereby hindering effective policy learning.

\subsection{Evolutionary Module}

To address the structural limitations of autoregressive policies, an evolutionary module based on GAs is introduced. This module provides a gradient-free optimization framework capable of generating improved solutions. Specifically, GAs iteratively evolve a \textit{population} composed of candidate solutions, each regarded as an \textit{individual}. At each generation, a subset of parent population is selected according to \textit{fitness} values, with higher-quality solutions assigned proportionally greater fitness. After selection, \textit{crossover} operators combine parts of two parent solutions to generate offspring that inherit structural traits from both; \textit{mutation} operators, in contrast, introduce small, localized changes to individual solutions. Based on these properties of GAs, the evolutionary module confers two key advantages: (1) genetic operators perturb entire solutions holistically rather than assembling them incrementally, thereby circumventing the detrimental influence of suboptimal partial trajectories; (2) by operating at the solution level, these operators preserve structural integrity and remain intrinsically resilient to issues arising from delayed or sparse rewards.

Despite these advantages, GAs exhibit notable limitations: (1) they lack explicit gradient-based learning signals, relying instead on heuristic selection and elimination to improve population fitness; (2) the evolved population is instance-specific and requires reinitialization for each new problem instance. In contrast, autoregressive policies, with their strong generalization across instances and efficient gradient-based adaptation, can accelerate and stabilize the search dynamics of GAs.

These properties motivate a hybrid design that integrates evolutionary exploration with gradient-based policy optimization. In the following section, we describe integration of two components.

\subsection{Evolutionary Augmentation Mechanism}

To harness the complementarity between autoregressive policies and global evolutionary search, we propose the Evolutionary Augmentation Mechanism (EAM), a closed-loop framework integrating evolution and policy learning. At each training iteration, with a predefined probability, a batch of trajectories $\mathrm{P}_\theta$ is sampled from the L2C-based policy network and used as the initial population for the GA, which evolves over $K$ generations to produce $\mathrm{P}^{(K)}$. The combined set $\mathrm{P}_\theta \cup \mathrm{P}^{(K)}$ is then used to update the policy via RL. 

EAM creates a synergistic mechanism that integrates \textit{evolution-guided policy learning} with \textit{policy-acclerated evolutionary search}.
\textbf{1) From the RL perspective}, GA-driven perturbations in EAM act as exogenous structural augmentation, exposing the policy to high-quality samples that deviate from its autoregressive patterns. These augmentations provide strong learning signals that guide the policy beyond its current solution pattern. Moreover, since GAs optimize complete trajectories directly, EAM inherently mitigates the sparsity of reward signals—particularly beneficial during early RL training stages. \textbf{2) From the GA perspective}, EAM initializes its population from the policy-generated trajectories, avoiding inefficiency of random initialization. As training progresses, the genetic algorithm becomes increasingly effective at discovering globally better solutions, as the policy network provides higher-quality initial populations to guide evolutionary search.

To ensure that evolved trajectories are both feasible and efficient, we choose and design genetic operators under two key criteria: (1) they must satisfy the structural constraints of the problem (e.g., tour completeness in TSP, capacity limits in CVRP), and (2) they must be computationally lightweight to support frequent execution during training. Specifically:

\begin{itemize}[leftmargin=*, itemsep=0pt, topsep=2pt]
    \item \textbf{Selection Operator.} To ensure efficient computation and maintain optimization effectiveness, we adopt an \textit{elitist selection strategy}~\citep{de1975analysis} in which crossover and mutation are applied only to a selected subset of top-performing individuals. This design minimizes the number of costly fitness evaluations while directing the search toward high-quality regions of the solution space.
    \item \textbf{Crossover Operator.} We employ \textit{Order Crossover} (OX)~\citep{davis1991handbook} for all problems studied in this work, due to its natural compatibility with node sequence representations. OX preserves a contiguous subsequence from one parent and fills the remainder according to the relative order in the second parent, ensuring structural feasibility. Compared to alternatives like \textit{Partially Mapped Crossover} (PMX)~\citep{goldberg1985alleleslociand}, which requires auxiliary mapping structures, OX provides better runtime efficiency and finer-grained control over perturbation locality.
    \item \textbf{Mutation Operator.} We use task-specific mutation operators. In TSP, CVRP, and PCTSP, we apply the classical \textit{2OPT}~\citep{croes1958method} edge exchange to remove edge crossings and shorten tour length with linear complexity. For OP, we design a heuristic node substitution strategy that replaces low-reward nodes with high-value candidates under a route length budget. These designs reflect our principle of maintaining feasibility while introducing meaningful variations that guide policy learning.
\end{itemize}

\subsection{Theoretical Foundations and Practical Implementation}
\label{sec:theorem}
EAM is designed as a general-purpose, lightweight, and plug-and-play enhancement module intended to improve learning efficiency and solution quality. It can be seamlessly integrated into mainstream L2C frameworks without necessitating any modifications to the original model architecture or optimization objectives. This seamless integration demands strong compatibility from EAM. However, by introducing evolutionary perturbations during policy training, EAM violates the core assumptions underpinning the REINFORCE algorithm, which is commonly used in the L2C methods.
Specifically, standard REINFORCE~\citep{williams1992simple} assumes that all training trajectories are sampled from the current policy distribution $p_\theta$, ensuring unbiased gradient estimates (i.e. $\mathbb{E}_{\boldsymbol{\boldsymbol{\tau}} \sim p_\theta}\big[ \nabla_\theta \log p(\boldsymbol{\boldsymbol{\tau}}) R(\boldsymbol{\boldsymbol{\tau}}) \big] = \nabla_\theta J(\boldsymbol{\tau})$). However, EAM applies genetic perturbations (crossover, mutation) to trajectories sampled from $p_\theta$ (i.e. population distribution at the $0$-th generation $p(\boldsymbol{\tau}_0) = p_\theta$), yielding an population distribution at the $K$-th generation $p(\boldsymbol{\tau}_K) \neq p_\theta$ and introducing bias: $\mathbb{E}_{\boldsymbol{\boldsymbol{\tau}} \sim p{(\boldsymbol{\boldsymbol{\tau}}_K)}}\big[ \nabla_\theta \log p(\boldsymbol{\boldsymbol{\tau}}) R(\boldsymbol{\boldsymbol{\tau}}) \big] \neq \nabla_\theta J(\boldsymbol{\tau})$. To quantify the effect of such distributional shifts, we model the gradient bias using the KL divergence between $p(\boldsymbol{\tau}_K)$ and $p_\theta$ (i.e., $p(\boldsymbol{\tau}_0)$). The following theorem provides an upper bound:

\begin{theorem}[Policy Gradient Difference Upper Bound]
\label{thm:main}
Let $p(\boldsymbol{\tau}_K)$ be the distribution after $K$ generations of evolution with crossover rate $\alpha$, mutation rate $\beta$, and selection rate $\rho$. Then:
\begin{align}
    \Big\|\mathbb{E}_{\boldsymbol{\tau}_K \sim p(\boldsymbol{\tau}_K)}\big[\nabla J(\boldsymbol{\tau}_K)\big] - \mathbb{E}_{\boldsymbol{\tau}_0 \sim p(\boldsymbol{\tau}_0)}\big[\nabla J(\boldsymbol{\tau}_0)\big]\Big\|_2 \leq \sqrt{2 D_{\text{KL}}(p(\boldsymbol{\tau}_K)\|p(\boldsymbol{\tau}_0))},
\end{align}
\begin{align}
    D_{\text{KL}}&(p(\boldsymbol{\tau}_K)\|p(\boldsymbol{\tau}_0)) \leq \nonumber\\ &\rho K \Bigg(\alpha \mathbb{E}_{p(f_{\text{cross}})}\bigg[\log \frac{\max{p(r_{\text{cross}}|f_{\text{cross}})}}{\min p(r_{\text{cross}}|f_{\text{cross}})}\bigg] + \beta \mathbb{E}_{p(f_{\text{mutate}})}\bigg[\log\frac{\max p(r_{\text{mutate}}|f_{\text{mutate}})}{\min p(r_{\text{mutate}}|f_{\text{mutate}})}\bigg]\Bigg),
\end{align}

\smallskip
\noindent
where $f_{\text{mutate}}$ and $f_{\text{cross}}$ denote the preserved fragments during mutation and crossover, respectively, while $r_{\text{mutate}}$ and $r_{\text{cross}}$ denote the filled-in segments generated to complete the solutions. Please refer to Appendix~\ref{app:proof} for the proof.

\end{theorem}

This result highlights the need to constrain perturbation strength to maintain policy stability. Guided by this, we adopt a task-aware evolutionary hyperparameter selection strategy based on KL divergence:
\begin{itemize}[leftmargin=*, itemsep=0pt, topsep=2pt]
    \item \textbf{TSP, CVRP, and PCTSP}. Due to their high sensitivity to node sequences, even small perturbations may cause significant KL shifts. To mitigate such instability, we apply moderate and reasonably balanced selection and mutation rates to stabilize the trajectory distribution.
    \item \textbf{OP}. As OP requires selecting a subset of nodes under a path length constraint, it is inherently more difficult to solve than the other COPs studied in this work.
    Combined with its sparse reward structure, this makes the policy more prone to getting stuck in local optima. We increase heuristic mutation strength to escape such stagnation and reduce crossover frequency to maintain stability.

    \item \textbf{General Scheduling}. As Theorem~\ref{thm:main} suggests, as the policy becomes more deterministic, the ratios $\frac{\max p(r_{\text{cross}}|f_{\text{cross}})}{\min p(r_{\text{cross}}|f_{\text{cross}})}$ and $\frac{\max p(r_{\text{mutate}}|f_{\text{mutate}})}{\min p(r_{\text{mutate}}|f_{\text{mutate}})}$ gradually increase, leading to a larger KL divergence between evolutionary samples and the policy distribution. Therefore, we apply a simulated annealing schedule to progressively decay genetic operation probabilities and gradually fade out perturbations in the later stages of training to ensure stability.
\end{itemize}

This KL-guided adaptation improves sample quality and exploration while preserving training stability across diverse COPs. Full task-aware settings are provided in Appendix~\ref{app:implement_details}.

\section{Experiments}

\subsection{Experimental Settings}
\label{sec:experment_settings}
\textbf{Tasks.} We evaluate EAM on four representative COPs: TSP, CVRP, PCTSP, and OP (see Section~\ref{sec:preliminaries}). EAM is integrated into AM~\citep{kool2018attention}, POMO~\citep{kwon2020pomo}, and Sym-NCO~\citep{kim2022sym} for TSP and CVRP, and into AM alone for PCTSP and OP, due to model compatibility and hardware constraints. TSP and CVRP with $N = 50$ and $100$ are used to evaluate cross-model and cross-scale generalization. PCTSP and OP test EAM's performance on tasks with partial routing and complex constraints.

\textbf{Baselines.} We compare against both traditional solvers and DRL-based methods. For TSP, we consider Concorde~\citep{applegate2006traveling}, LKH-3~\citep{helsgaun2017extension}, AM, POMO, and Sym-NCO; for CVRP, we use HGS~\citep{vidal2022hybrid}, LKH-3, AM, POMO, and Sym-NCO; for PCTSP, we benchmark against ILS~\citep{lourencco2018iterated} and AM; and for OP, we compare with Compass~\citep{kobeaga2018efficient} and AM.

\textbf{Metrics.} We evaluate average performance and total inference time cost on $10k$ test instances. For neural solvers, we report results under three decoding strategies: greedy, multi-start, and data augmentation, to reflect trade-offs between quality and latency. The multi-start decoding strategy generates multiple trajectories on the same COP instance, improving performance compared to greedy decoding but incurring additional inference overhead. The data augmentation decoding strategy enhances COP instances by generating multiple symmetric variants, further boosting performance at the cost of even higher inference overhead. To highlight differences in practical deployment suitability, traditional solvers run on CPU, neural methods on GPU. 

For more details on experimental settings, please refer to Appendix~\ref{app:detailed_exp_settings}.

\subsection{Experimental Results}
\subsection{TSP and CVRP}

\begin{table}[t]
  \centering
  
  \caption{\textbf{Performance on TSP and CVRP.} Bold represents the best performances in each task. `-' indicates that the solver does not support the problem. `s' indicates multi-start sampling, `$8\times$' indicates data augmentation for 8 times. The number of multi-start initial points set equal to the number of nodes in the problem instance. We use LHK3's performance on CVRP from ~\citep{sun2024learning}. }
  \begin{adjustbox}{width=\linewidth}
  \small
  \begin{tabular}{llcccccc cccccc}
    \specialrule{0.8pt}{0pt}{0pt}
    & \textbf{Method} 
    & \multicolumn{3}{c}{TSP50} 
    & \multicolumn{3}{c}{TSP100} 
    & \multicolumn{3}{c}{CVRP50} 
    & \multicolumn{3}{c}{CVRP100} \\
    \cmidrule(lr){3-5} \cmidrule(lr){6-8} \cmidrule(lr){9-11} \cmidrule(lr){12-14}
    & & Cost $\downarrow$ & Gap $\downarrow$ & Time $\downarrow$ 
  & Cost $\downarrow$ & Gap $\downarrow$ & Time $\downarrow$ 
  & Cost $\downarrow$ & Gap $\downarrow$ & Time $\downarrow$ 
  & Cost $\downarrow$ & Gap $\downarrow$ & Time $\downarrow$ \\
    \specialrule{0.8pt}{0pt}{0pt}

    & Concorde       & 5.690 & --     & 13.3m  & 7.761 & --     & 40.8m &      & --     &      &      & --     &  \\
    & HGS            &     & --     &     &     & --     &     & 10.366 & --     & 2.1h & 15.586 & --     & 3.6h \\
    & LKH3           & 5.690 & 0.00\% & 2.0m  & 7.761 & 0.00\% & 18.9m & 10.367 & 0.01\% & 1h    & 15.667 & 0.52\% & 2h \\
    \midrule

    & AM (greedy. )            & 5.794 & 1.82\% & 2s    & 8.123 & 4.67\% & 5s    & 10.949 & 5.63\% & 2s    & 16.580 & 6.38\% & 5s \\
    & EAM-AM (greedy. )          & 5.780 & 1.57\% & 2s    & 8.071 & 4.00\% & 5s    & 10.893 & 5.08\% & 2s    & 16.555 & 6.22\% & 5s \\
    \midrule

    & POMO (s.)                 & 5.702 & 0.21\% & 9s    & 7.800 & 0.50\% & 19s   & 10.555 & 1.82\% & 12s   & 15.894 & 1.98\% & 28s \\
    & EAM-POMO (s.)             & 5.702 & 0.20\% & 10s   & 7.796 & 0.45\% & 19s   & 10.500 & 1.29\% & 13s   & 15.859 & 1.75\% & 29s \\
    & SymNCO (s. )      & 5.698 & 0.14\% & 10s   & 7.793 & 0.42\% & 19s   & 10.495 & 1.24\% & 13s   & 15.839 & 1.63\% & 29s \\
    & EAM-SymNCO (s. )  & 5.698 & 0.14\% & 10s   & 7.777 & 0.21\% & 19s   & 10.481 & 1.11\% & 13s   & 15.826 & 1.54\% & 29s \\
    \midrule
    & POMO (s. $8\times$)        & 5.697 & 0.11\% & 15s   & 7.775 & 0.18\% & 64s   & 10.463 & 0.94\% & 47s   & 15.760 & 1.12\% & 306s \\
    & EAM-POMO (s. $8\times$)    & 5.693 & 0.05\% & 15s   & \textbf{7.772} & \textbf{0.14\%} & 63s   & \textbf{10.422} & \textbf{0.54\%} & 47s   & \textbf{15.733} & \textbf{0.95\%} & 306s \\
    & SymNCO (s. $8\times$)      & 5.698 & 0.13\% & 15s   & 7.775 & 0.18\% & 63s   & 10.438 & 0.69\% & 48s   & 15.818 & 1.49\% & 306s \\
    & EAM-SymNCO (s. $8\times$)  & \textbf{5.692} & \textbf{0.04\%} & 16s   & 7.774    & 0.17\% & 64s   & 10.427 & 0.59\% & 48s   & 15.806 & 1.42\% & 306s \\
    \specialrule{1pt}{0pt}{0pt}
    \label{tab:tsp_cvrp_results}
  \end{tabular}
  \end{adjustbox}
\end{table}

Table~\ref{tab:tsp_cvrp_results} shows that EAM consistently enhances solution quality across models, instance sizes, and decoding strategies. On CVRP50, EAM-POMO reduces the optimality gap from $0.94\%$ to $0.54\%$, while on TSP100, EAM-POMO achieves a $0.14\%$ gap compared to $0.50\%$ from the original POMO. EAM-SymNCO further achieves the best overall performance on TSP50 with a $0.04\%$ gap. These improvements are obtained without any additional inference time. Moreover, as shown in Figure~\ref{fig:tsp100_cvrp100}, EAM significantly accelerates convergence during training on most tasks, indicating its effectiveness not only in final performance but also in training efficiency.

\begin{figure}[H]
  \centering

  \begin{subfigure}[b]{0.32\linewidth}
    \includegraphics[width=\linewidth]{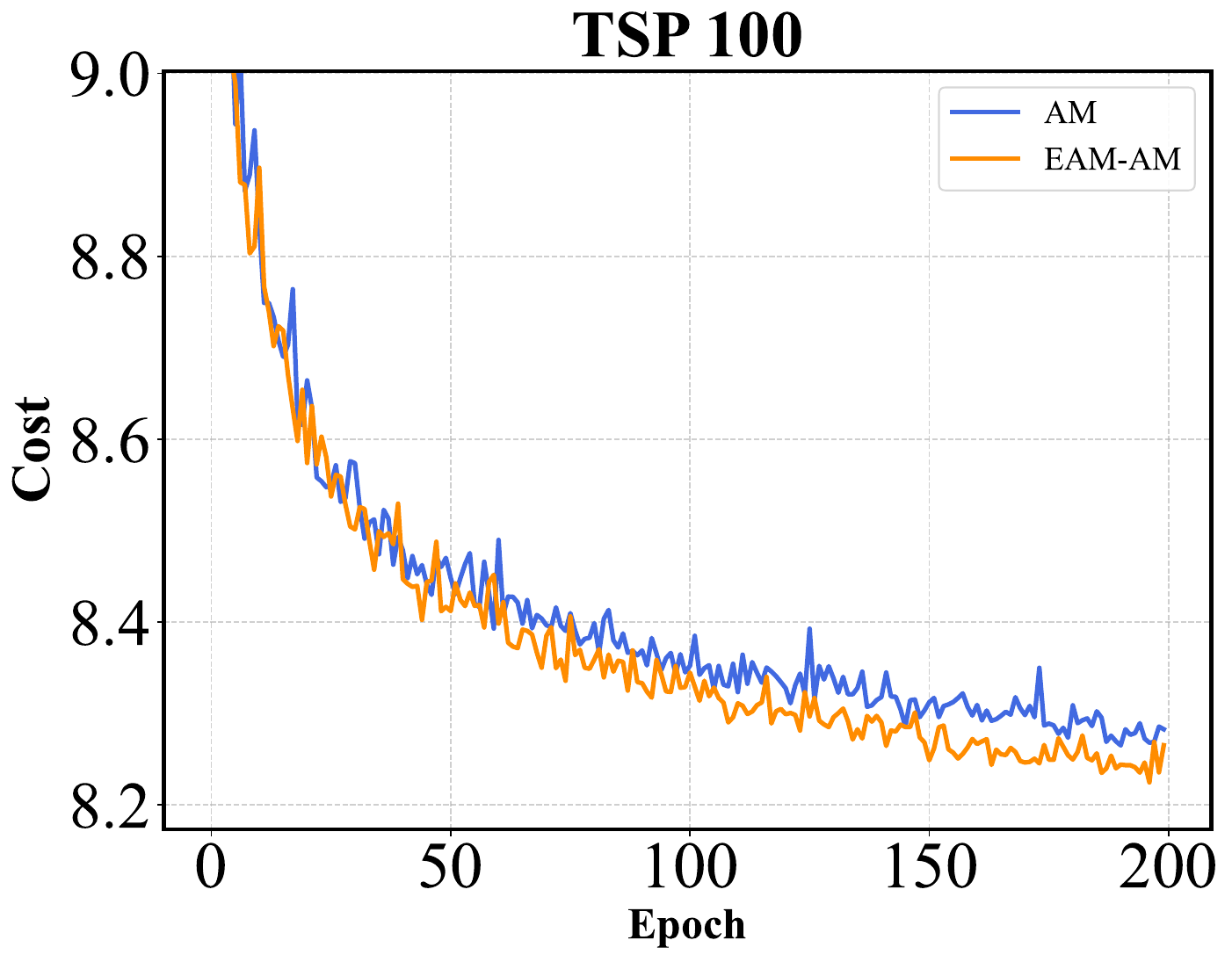}
  \end{subfigure}
  \hfill
  \begin{subfigure}[b]{0.32\linewidth}
    \includegraphics[width=\linewidth]{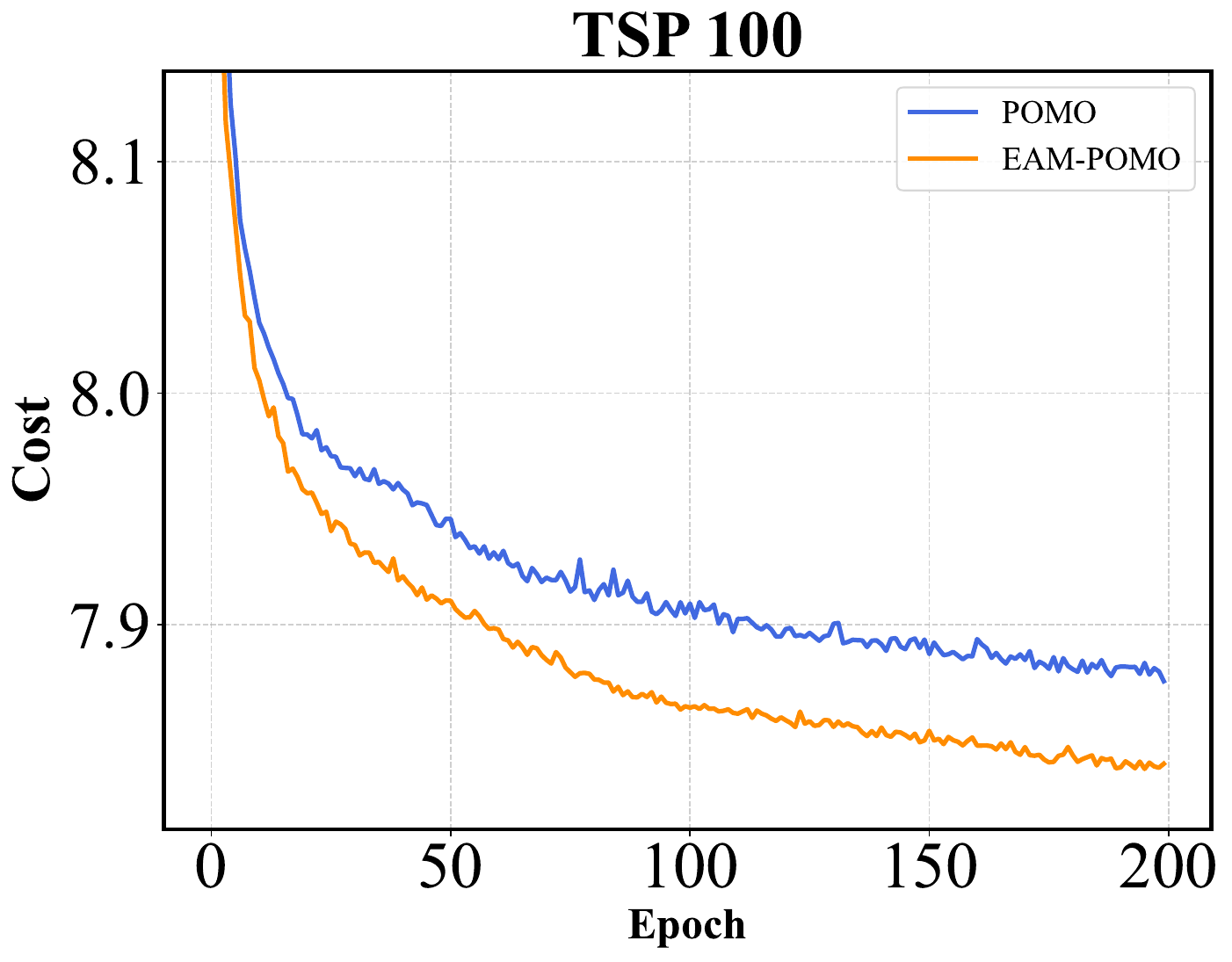}
  \end{subfigure}
  \hfill
  \begin{subfigure}[b]{0.32\linewidth}
    \includegraphics[width=\linewidth]{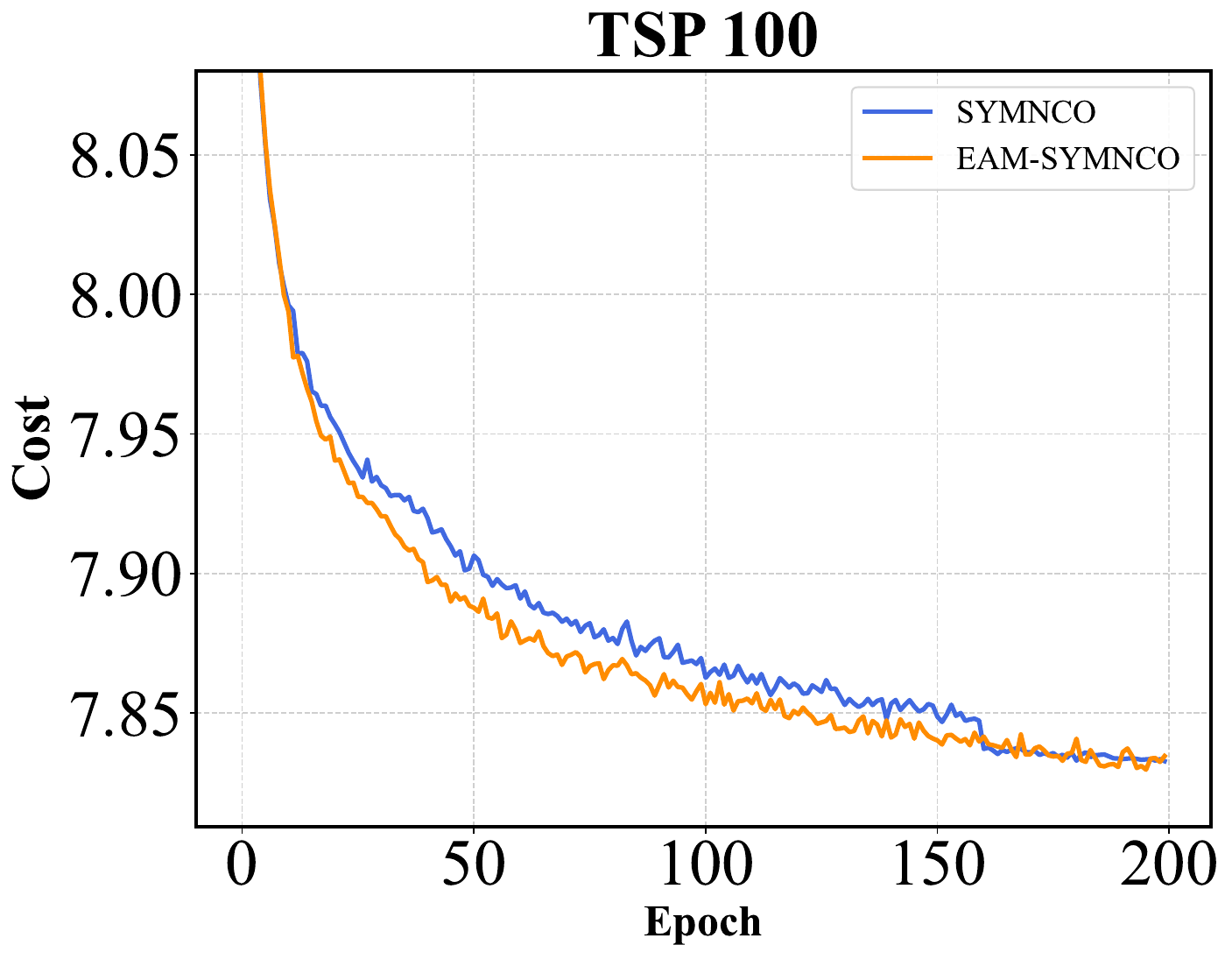}
  \end{subfigure}

  \begin{subfigure}[b]{0.32\linewidth}
    \includegraphics[width=\linewidth]{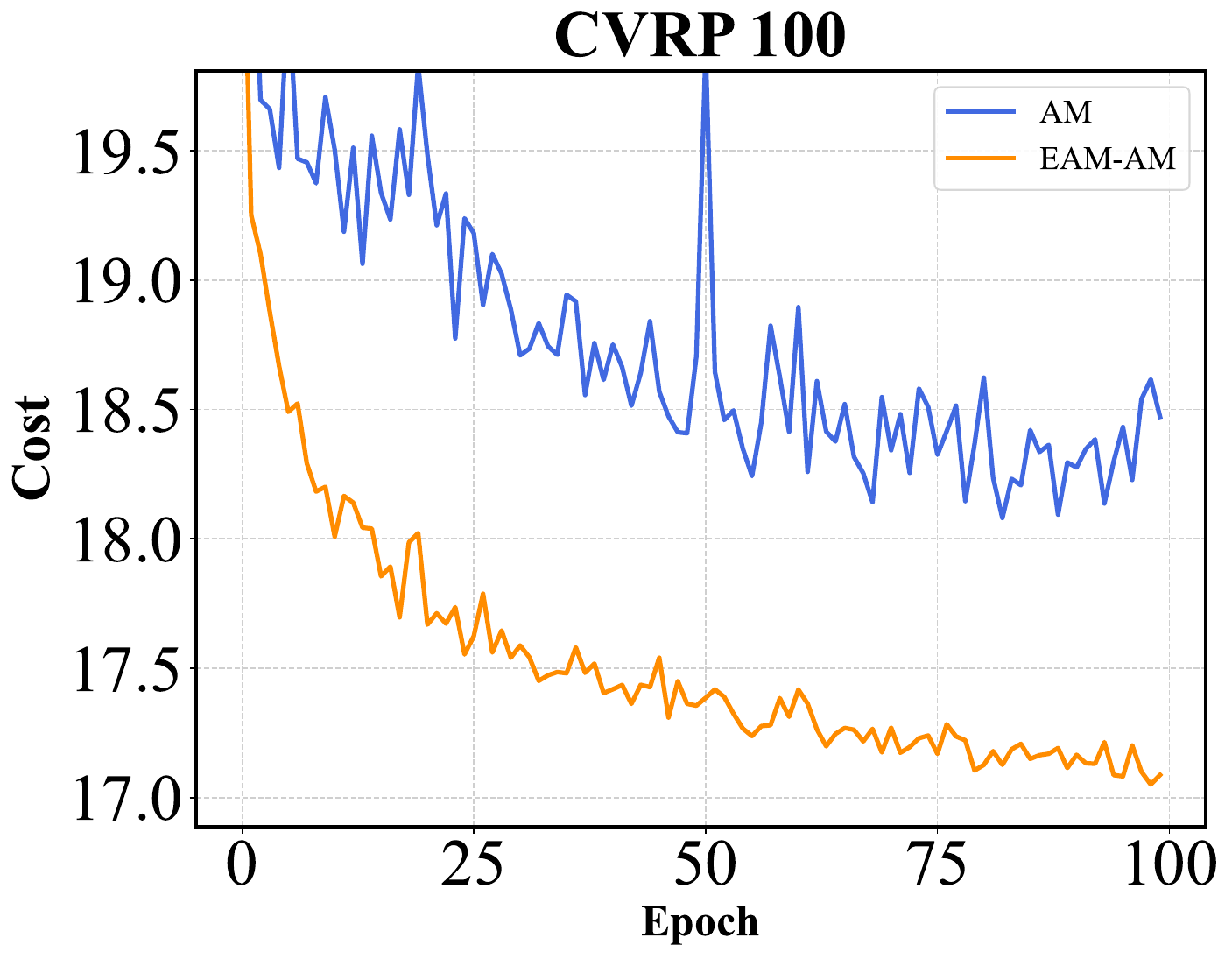}
  \end{subfigure}
  \hfill
  \begin{subfigure}[b]{0.32\linewidth}
    \includegraphics[width=\linewidth]{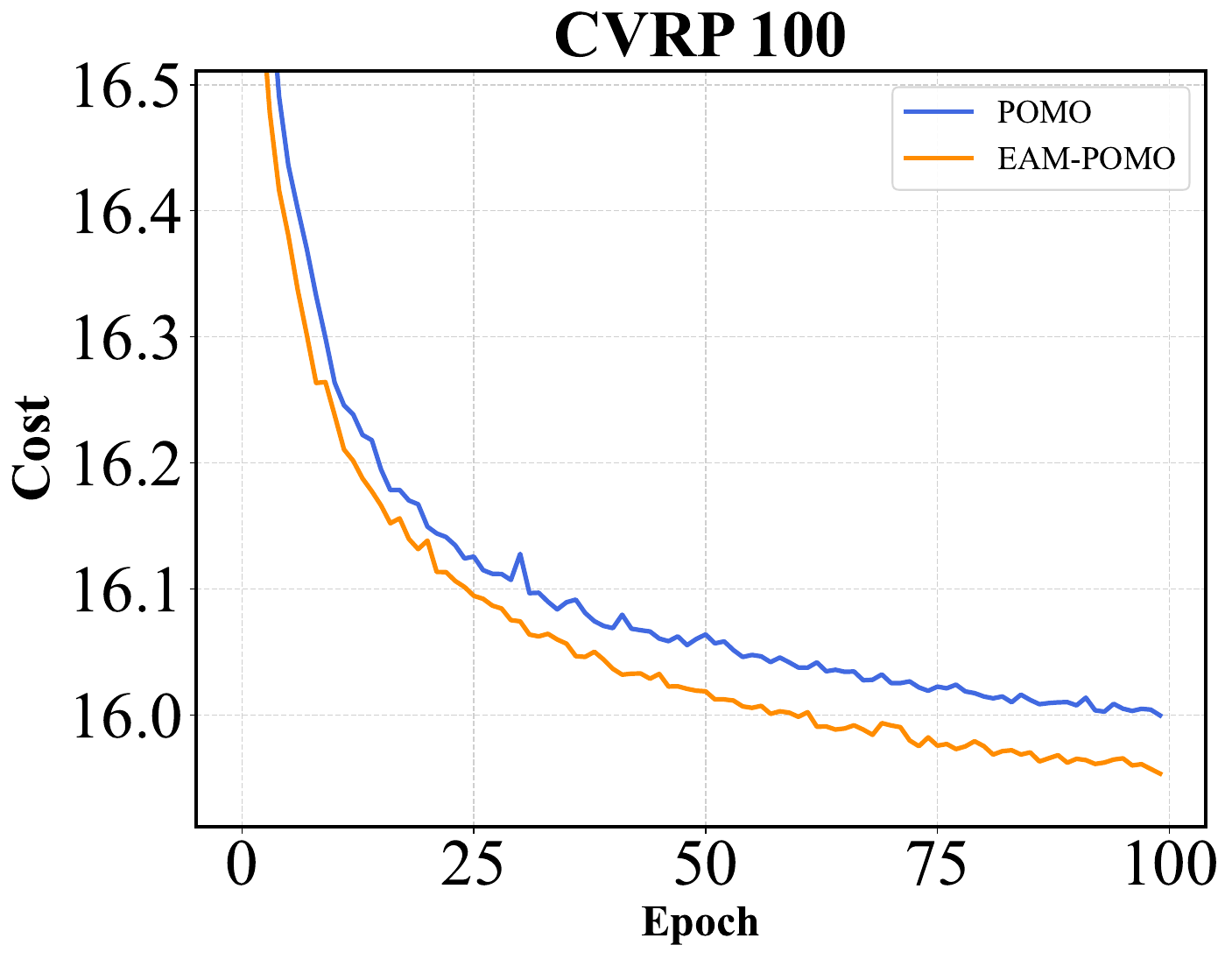}
  \end{subfigure}
  \hfill
  \begin{subfigure}[b]{0.32\linewidth}
    \includegraphics[width=\linewidth]{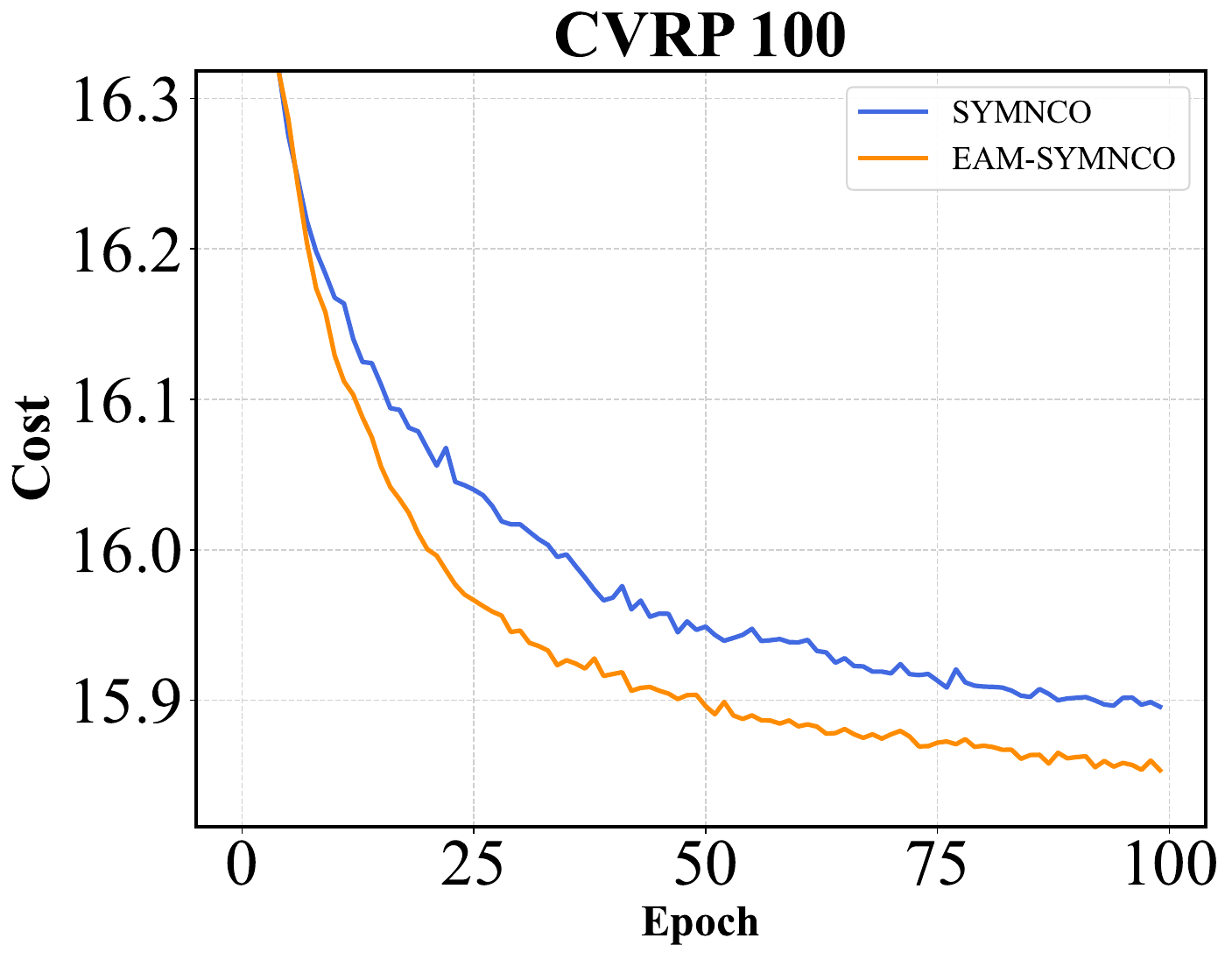}
  \end{subfigure}

  \caption{\textbf{Training curves of AM, POMO, and Sym-NCO with and without EAM on TSP-100 and CVRP-100.} EAM consistently accelerates convergence across different backbones and problem settings, highlighting its ability to improve training efficiency in addition to final solution quality.}
    \label{fig:tsp100_cvrp100}
\end{figure}

\subsubsection{PCTSP and OP}
As shown in Table~\ref{tab:pctsp_op_results}, EAM also exhibits consistent performance gains on PCTSP and OP. 
For example, on OP, EAM-AM improves the objective from 31.54 to 31.85, narrowing the gap from $4.97\%$ to $4.03\%$. These improvements highlight EAM’s ability to enhance the exploratory capacity of the policy in scenarios characterized by partial routing and complex objective functions, further validating its generality and adaptability as a policy enhencement module across diverse COPs.

\begin{table}[htbp]
  \centering
  \caption{\textbf{Performance on PCTSP and OP}. Notations are the same with Table~\ref{tab:tsp_cvrp_results}. We use ILS and Compass's performance from ~\citep{kim2022sym}.}
  \begin{adjustbox}{width=0.6\linewidth}
  \small
  \begin{tabular}{lccccccc}
    \specialrule{1.2pt}{0pt}{0pt}
    \multirow{2}{*}{\textbf{Method}} 
      & \multicolumn{3}{c}{\textbf{PCTSP100}} 
      & \multicolumn{3}{c}{\textbf{OP}} \\
    \cmidrule(lr){2-4} \cmidrule(lr){5-7}
    & Cost $\downarrow$ & Gap $\downarrow$ & Time $\downarrow$ 
    & Obj $\uparrow$ & Gap $\downarrow$ & Time $\downarrow$ \\
    \specialrule{1.2pt}{0pt}{0pt}

    ILS C++   & 5.98  & --     & 12h   &     & --     &  \\
    Compass   &     & --     &     & 33.19 & --     & 15m \\
    \midrule
    AM (greedy.)       & 6.21  & 3.92\% & 2s    & 31.54 & 4.97\% & 2s \\
    EAM-AM (greedy.)   & \textbf{6.20}  & \textbf{3.61\%} & 2s    
              & \textbf{31.85} & \textbf{4.03\%} & 2s \\
    \specialrule{1.2pt}{0pt}{0pt}
  \end{tabular}
  \end{adjustbox}
  \label{tab:pctsp_op_results}
\end{table}

\subsubsection{Ablation Study}
We conduct a comprehensive ablation study to evaluate the contribution of key designs in EAM. The experiments are organized along two dimensions: 1) evolutionary hyperparameter configuration and 2) annealed evolution frequencies. For each aspect, we systematically alter specific design components to evaluate their influence on solution quality and training dynamics.

\begin{figure}[htbp]
  \begin{minipage}[b]{0.32\linewidth}
    \centering
    \includegraphics[width=\linewidth]{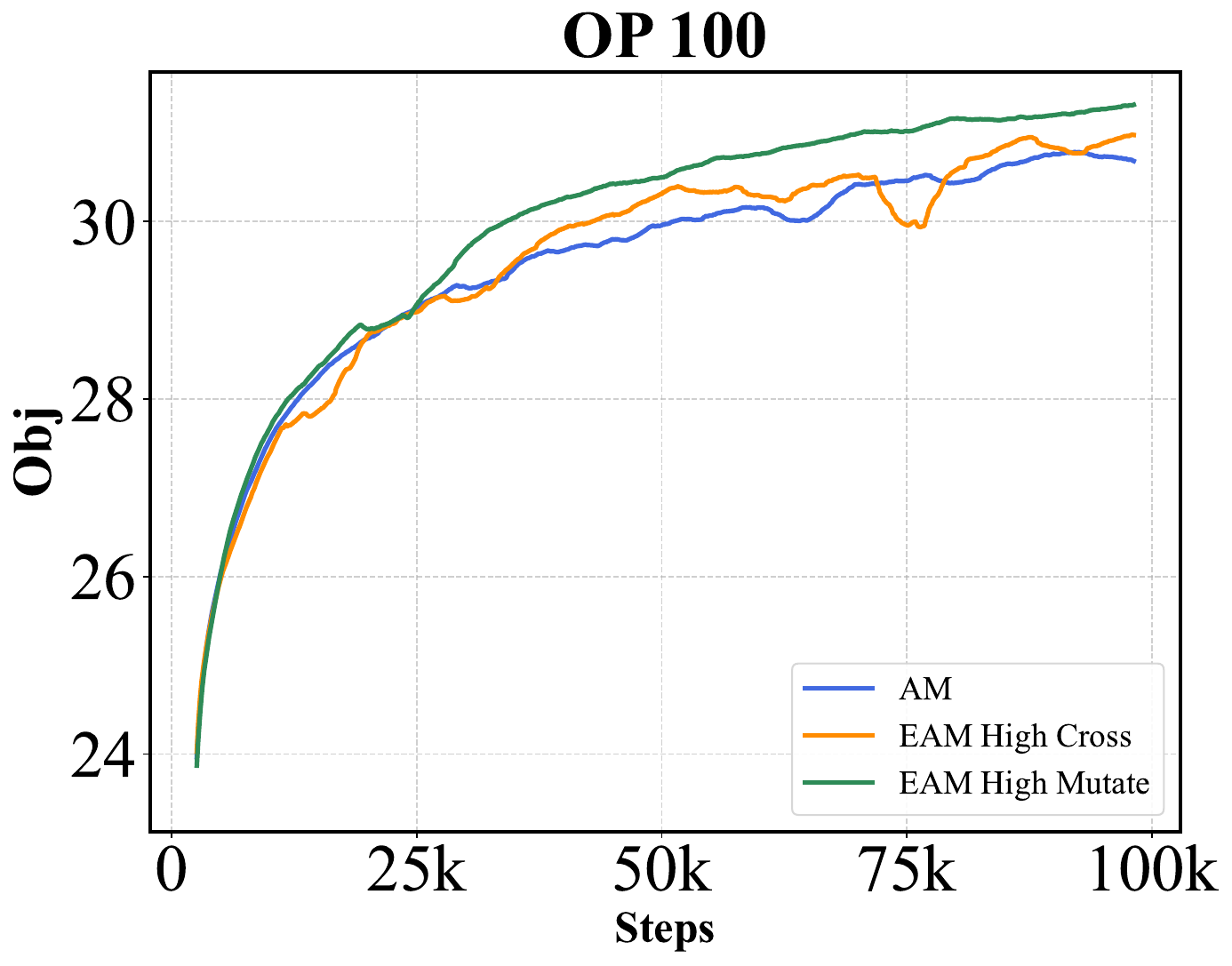}
    \caption{Training curves of various hyperparameters.}
    \label{fig:abli_param}
  \end{minipage}
  \hfill
  \begin{minipage}[b]{0.32\linewidth}
    \centering
    \includegraphics[width=\linewidth]{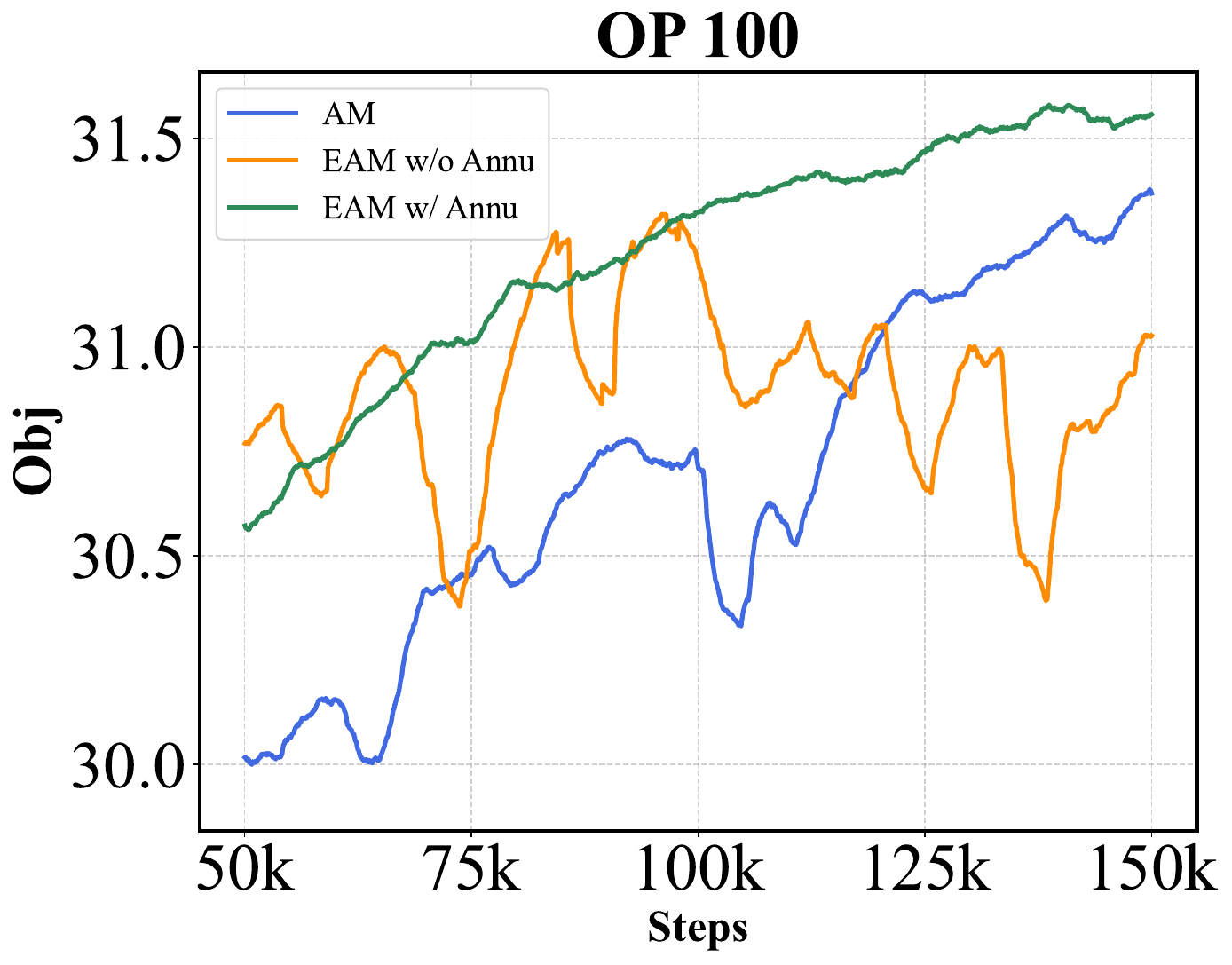}
    \caption{Training curves of whether to use annealing.}
    \label{fig:abli_annu}
  \end{minipage}
  \hfill
  \begin{minipage}[b]{0.32\linewidth}
    \centering
    \includegraphics[width=\linewidth]{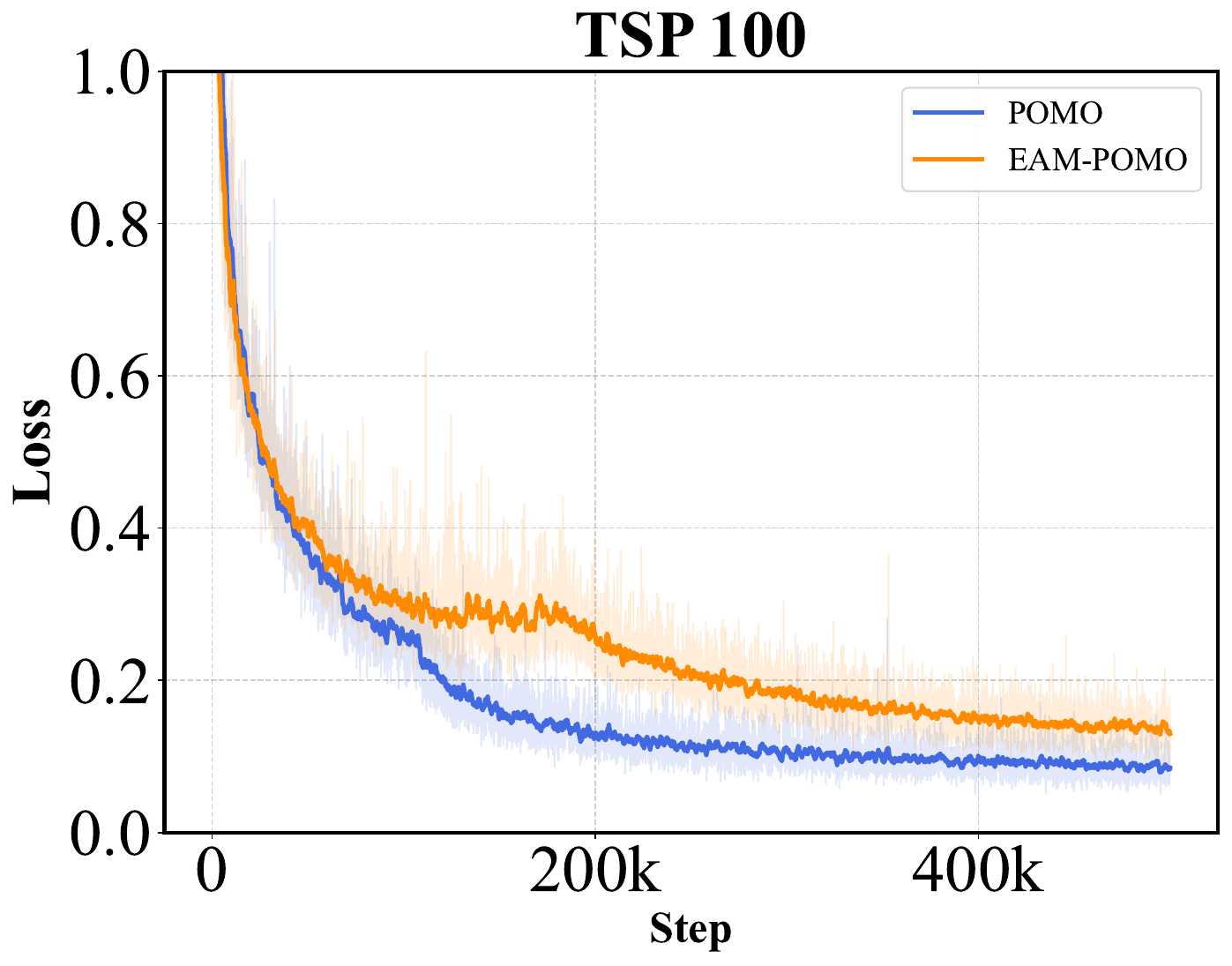}
    \caption{Training curves of absolute loss on TSP 100.}
    \label{fig:tsp100_loss}
  \end{minipage}
\end{figure}

\paragraph{Evolutionary Hyperparameter Configuration.}
We examine the impact of different crossover and mutation rate configurations in EAM on OP (Fig.~\ref{fig:abli_param}). Specifically, we compare two settings: (1) high mutation rate with low crossover rate, and (2) low mutation rate with high crossover rate. Empirical results show that the first setting leads to faster convergence and superior final solution quality. This configuration aligns better with the OP objective, where the absolute node order is largely irrelevant and aggressive node replacement helps uncover high-reward substructures more effectively.

Additionally, this configuration results in a smaller KL divergence upper bound (see Appendix~\ref{app:implement_details}), implying reduced bias when incorporating evolved trajectories into policy updates. This observation supports the theoretical rationale for modeling KL divergence and highlights the importance of task-aware perturbation strength in maintaining training stability.

\paragraph{Annealed Evolution Frequency.}

We compare constant and annealed evolution frequencies (Fig.~\ref{fig:abli_annu}). Disabling annealing leads to slower improvement and significant instability, indicating that the utility of evolutionary perturbations gradually shifts toward negative effects as training progresses. These results validate the annealing schedule and align with the theoretical insight in Section~\ref{sec:theorem}.

\section{Discussion}
\subsection{Solution Visualization}

\begin{figure}[htbp]
    \centering
    \begin{subfigure}[t]{0.48\linewidth}
        \centering
        \includegraphics[width=\linewidth]{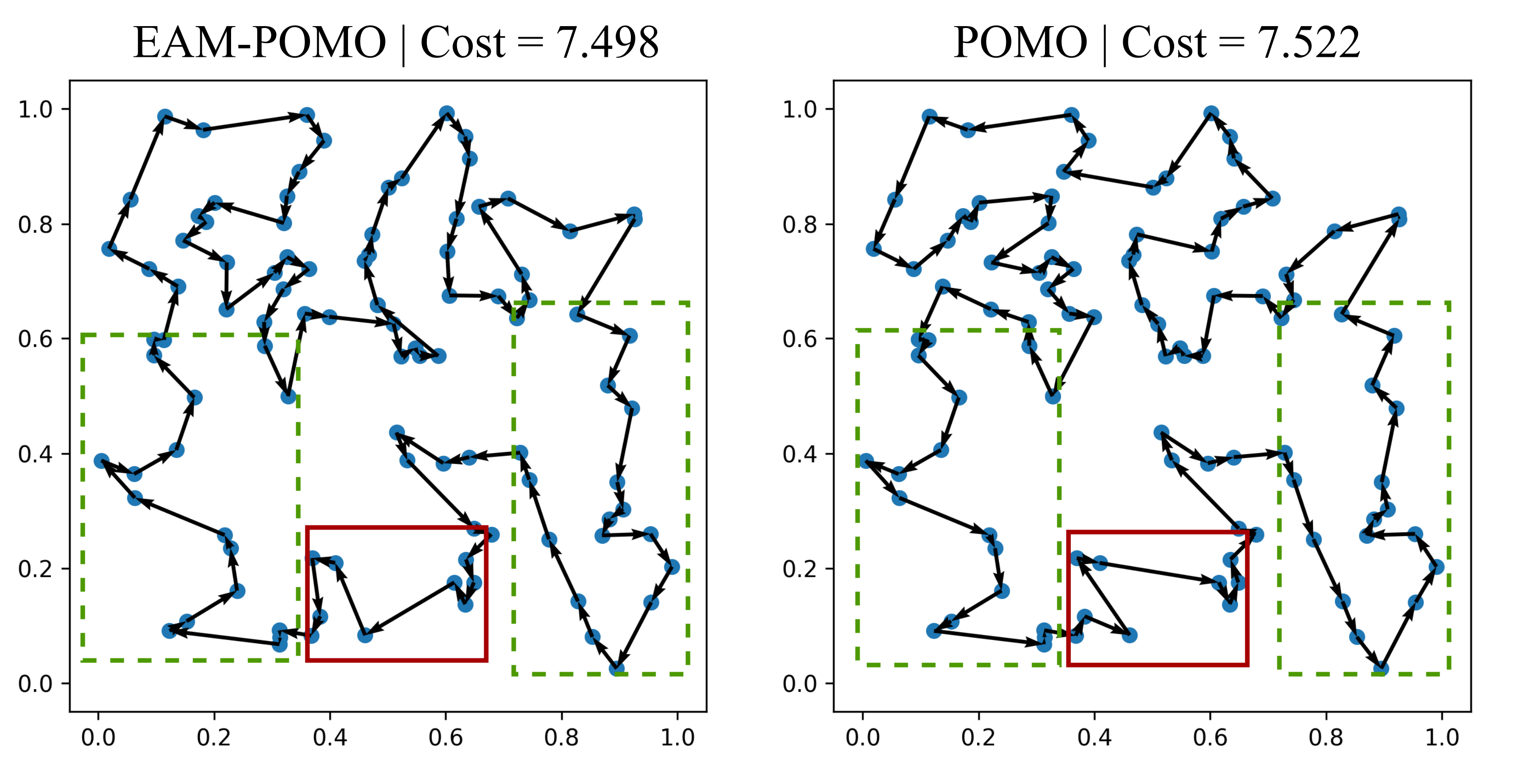}
        \caption{}  
        \label{fig:tsp100_visual_1}
    \end{subfigure}
    \hfill
    \begin{subfigure}[t]{0.48\linewidth}
        \centering
        \includegraphics[width=\linewidth]{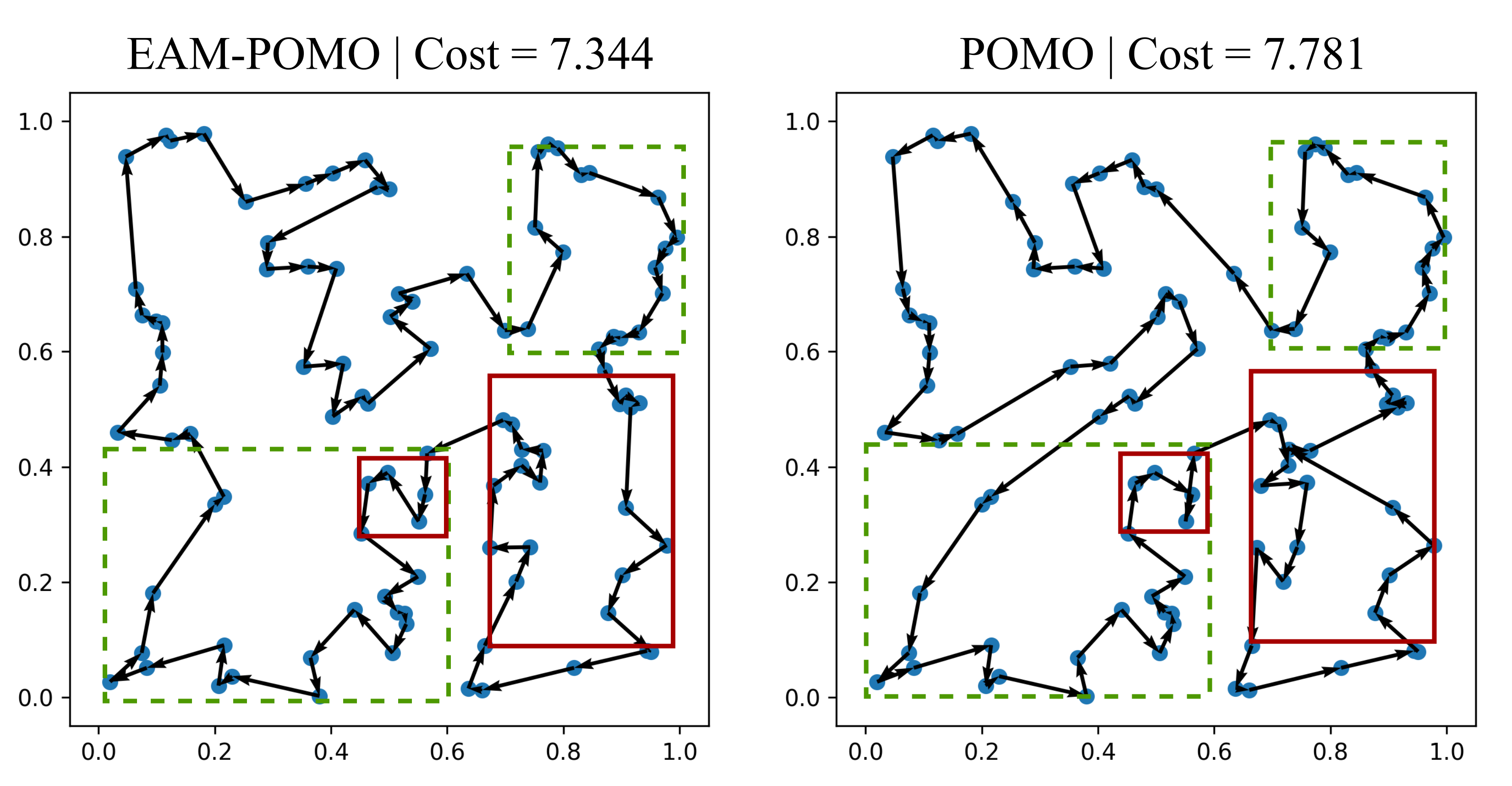}
        \caption{}  
        \label{fig:tsp100_visual_2}
    \end{subfigure}
    \caption{\textbf{Solutions generated by POMO and EAM-POMO on representative TSP-100 instances.} Both models produce globally similar tours, but EAM-POMO exhibits finer local refinements—especially in dense regions—demonstrating the structural optimization effect of EAM.}
    \label{fig:tsp100_visual}
\end{figure}

To examine how EAM guides the policy network toward generating higher-quality solutions during inference, we select representative instances on TSP-100 to compare the solution trajectories produced by POMO and EAM-POMO (Fig.~\ref{fig:tsp100_visual}). Both models exhibit highly similar structures, indicating that EAM does not fundamentally alter the overall behavior of the policy.

However, due to the encoder-decoder architecture commonly adopted in L2C methods, node embeddings are typically computed in a static manner before trajectory construction.As a result, spatially adjacent nodes often receive similar embeddings, making it difficult for the policy network to effectively distinguish between them in high-density regions. This limitation can lead to suboptimal node selection and inaccurate ordering in  solution. In contrast, genetic algorithms iteratively refine solutions through crossover and mutation, enabling fine-grained adjustments to the tour structure and compensating for the policy network’s limited capacity to capture subtle local variations. The red-boxed region in Figure~\ref{fig:tsp100_visual} illustrates that the observed performance difference is mainly driven by the discrepancy in inference quality between EAM-POMO and POMO.

\subsection{Training Stability Analysis}

To empirically validate the KL-based theoretical analysis proposed in this work, we further analyze the training dynamics on TSP-100 by tracking the absolute value of the policy gradient loss with and without EAM (Fig.~\ref{fig:tsp100_loss}). Under REINFORCE with a normalized advantage, the loss magnitude serves as a proxy for the log-probability assigned by the policy to sampled trajectories: lower log-probability leads to higher absolute value of the loss.

Empirical results show that EAM consistently induces higher loss values, indicating that evolved samples lie in low-probability regions of the policy distribution. At the beginning of training, however, the loss values under EAM and the baseline are nearly identical. This empirical observation consistent with the theoretical prediction in Theorem~\ref{thm:main}, which states that when the trajectories approximately follow a uniform distribution, $\frac{\max p(r_{\text{cross}}|f_{\text{cross}})}{\min p(r_{\text{cross}}|f_{\text{cross}})} \approx 1$ and $\frac{\max p(r_{\text{mutate}}|f_{\text{mutate}})}{\min p(r_{\text{mutate}}|f_{\text{mutate}})} \approx 1$, indicating $D_{\text{KL}} \approx 0$. As training progresses, the loss increase under EAM remains bounded, indicating that the policy is progressively absorbing the structural information encoded in the evolutionary samples. This observation aligns with Theorem~\ref{thm:main}, confirming that EAM introduces controlled distributional bias and supports stable policy updates.

Together with the earlier structural visualization, these results demonstrate the consistency between EAM’s theoretical foundations and its empirical behavior, reinforcing its effectiveness as a reliable enhancement for NCO.

\section{Related Works}

\subsection{Learning-to-Construct Methods for COPs}

L2C methods solving COPs by autoregressively constructing solutions via sequence models. Early works include the Pointer Network~\citep{vinyals2015pointer}, REINFORCE-based training~\citep{bello2016neural}. \citet{kool2018attention} proposed the Attention Model (AM), which leveraged the Transformer architecture to enhance the expressiveness of policy networks, and has since become a widely adopted baseline. Building upon AM, \citet{kwon2020pomo} exploited the inherent symmetry of problems such as TSP and CVRP by employing a multi-start sampling strategy, resulting in further performance gains. \citet{kim2022sym} extends symmetry principle by utilizing data augmentation and modifying the loss formulation to harness the problem’s inherent invariance more effectively. Subsequent research extending L2C methods can be broadly categorized into two lines of work. The first line of work emphasizes increasing the structural diversity and optimality of solutions~\citep{grinsztajn2023winner, hottung2021efficient, sun2024learning}, while the other line of work centers on the explicit modeling of structural priors and problem symmetries to speed up convergence~\citep{wang2024leader, son2024equity}. 

Although the aforementioned methods have achieved substantial performance gains in combinatorial optimization, they typically operate by conditioning on incrementally constructed partial solutions and lack explicit global search mechanisms, rendering them susceptible to local optima. In contrast, our EAM introduces an evolutionary augmentation mechanism that enables global trajectory exploration beyond the capabilities of standard autoregressive models.

\subsection{Hybrid RL and Heuristic Methods for COPs}

The introduction of the ERL framework~\citep{khadka2018evolution}, marked a turning point that catalyzed the development of a broad class of hybrid optimization methods combining evolutionary algorithms with RL~\citep{sigaud2023combining}. In the context of combinatorial optimization, the integration of heuristic algorithms and RL has manifested in several distinct algorithmic paradigms. One line of work leverages RL agents to inject solution diversity into the evolutionary population, facilitating more effective exploration across structurally constrained combinatorial landscapes~\citep{radaideh2021rule}. Another line of work focuses on learning adaptive mutation and crossover operators using RL, replacing manually designed heuristics with learned local search behaviors~\citep{hu2022constrained, liu2023neurocrossover}. A third line of work incorporates heuristic algorithms during inference to refine initial solutions sampled from policy networks, thereby enhancing the quality of final solutions at inference~\citep{greenberg2025accelerating}.

Among the various paradigms that combine heuristic algorithms with RL for solving COPs, RLHO~\cite{cai2019reinforcement} bears the closest architectural resemblance to our approach. In RLHO, a policy network is first employed to sample feasible solutions, which are then refined using simulated annealing. The improvement in solution quality achieved by the heuristic algorithm is treated as a reward and fed back to the policy network for training. Although this method also establishes a coupled training framework between RL and heuristics, its primary focus lies in training the policy to generate solutions that are amenable to heuristic improvement, thereby adopting a \textit{heuristic-centered} paradigm with RL playing a supporting role.

In contrast, the EAM proposed in this work adopts a \textit{RL–centered} framework with heuristic optimization serving as an auxiliary enhancement. Specifically, we apply a genetic algorithm to locally evolve policy-generated solution trajectories, thereby improving both their structural diversity and solution quality. These evolved trajectories are then directly integrated into the policy gradient optimization process. This approach maintains the high inference efficiency of RL while substantially enhancing exploration capability and accelerating training convergence.

\section{Conclusion and Future Work}
\label{sec:conclusion}
We propose EAM, a general framework that unifies the strengths of RL and GAs. EAM operates by injecting evolved samples generated by GA into the RL policy learning loop, thereby encouraging exploration and mitigating the limitations of sparse rewards and suboptimal trajectory construction. Theoretical analyses based on KL divergence provide guarantees on training stability, while empirical studies across diverse combinatorial optimization tasks—including TSP, CVRP, PCTSP, and OP—demonstrate consistent gains in both solution quality and training efficiency. Our results confirm EAM’s effectiveness as a lightweight, model-agnostic, and plug-and-play enhancement that can be readily integrated into existing DRL-NCO pipelines.

\textbf{Limitations and Future Work.}
While EAM demonstrates consistent improvements across multiple COPs, there are several limitations worth noting. First, our current evaluation focuses primarily on representative routing problems. Although these benchmarks are widely used and sufficiently diverse in structure, they do not fully capture the breadth of real-world COPs (e.g., Job Shop Scheduling, Knapsack, and Graph Coloring Problems). Extending EAM to a broader class of COPs with different structural constraints remains an important and worthwhile direction for future research. Furthermore, evaluating EAM under stochastic problem settings—such as constraints are subject to uncertainty~\citep{berto2024routefinder}—could provide deeper insights into its robustness and generalization capabilities.

Second, the problem instances studied in this paper are of moderate scale (typically up to 100 nodes). While this is consistent with prior DRL-based approaches, scaling to significantly larger instances remains an open challenge. We note, however, that many divide-and-conquer schemes~\citep{zheng2024udc, pan2023h} have been successfully applied to large-scale COPs by breaking them into smaller subproblems. EAM can be naturally embedded within such decomposition-based solvers to improve the quality of subproblem solutions and thereby enhance overall performance. We consider this a particularly compelling avenue for future exploration.

\bibliographystyle{plainnat}
\bibliography{references}


\newpage

\appendix
\newcommand{\GeneratePopulation}{GeneratePopulation} 
\newcommand{\red}[1]{\textcolor{red}{#1}}

\section{Proof of Theorem~\ref{thm:main}}
\label{app:proof}

\paragraph{Remark.} We use a unified sequence encoding for all problems.
Let $\mathcal{V}=\{0,1,\dots,n\}$ be the set of node indices, where
$0$ denotes the depot or a special
“dummy symbol” when a position in the sequence is unused.
For a fixed length $L\!\ge\!n$ we encode every candidate tour as a
sequence $\boldsymbol{\tau}=[\tau_1,\dots,\tau_L]\in\mathcal V^{\,L}$:

\begin{itemize}[leftmargin=1.4em]
\item \textbf{Complete‑path problems (TSP, CVRP).}  
  Each node $1,\dots,n$ appears exactly once in~$\boldsymbol{\tau}$, the depot~$0$
  may appear multiple times (CVRP).
  No dummies are present.

\item \textbf{Partial‑path problems (PCTSP, OP).}  
  The same alphabet $\mathcal V$ is used, but positions that are not
  visited in the tour are filled with the dummy symbol~$0$.
  Consequently $0$ may occur multiple times.
\end{itemize}

Because~$0\in\mathcal V$ is treated as an ordinary symbol in both
settings, we obtain

\begin{enumerate}[label=(\roman*),leftmargin=1.6em]
\item a common support
      $\mathcal X=\mathcal V^{\,L}$ for all generation distributions
      $p_k$, irrespective of whether the underlying problem is complete
      or partial;
\item strictly positive one‑step kernels
      $P_k(\tau,\cdot)$, since each genetic operator assigns
      non‑zero probability to every symbol in every position,
      including repetitions of~$0$ when they serve as dummies.
\end{enumerate}

Hence Lemmas~\ref{lem:exp_tv}–\ref{lem:crossover_measurable} and
Theorems~\ref{thm:divergence_k_generation}–\ref{thm:gradient_difference_short}
apply verbatim to both
complete tours (TSP, CVRP) and
partial tours (PCTSP, OP).

\paragraph{Facts (standard).}
Log--Sum, KL‑Convexity, Chain Rule, Data–Processing, Pinsker~\citep{stone2015information}.
\begin{lemma}[Expectation–Total Variation Bound]\label{lem:exp_tv}
Let $P,Q$ be probability measures on $(\mathcal X,\mathcal F)$
and $f:\mathcal X\to\mathbb R^d$ be bounded with
$\|f\|_{2,\infty}:=\sup_{x}\|f(x)\|_2<\infty$.
Then
\[
  \bigl\|\mathbb E_Q[f]-\mathbb E_P[f]\bigr\|_2
  \;\le\;
  2\,D_\mathrm{TV}(P,Q)\,\|f\|_{2,\infty}.
\]
\end{lemma}
\begin{lemma}[Terminal KL bound via step-wise KL to the initial distribution]
\label{lem:markov_chains_p0}
Let $X_0 \!\to\! X_1 \!\to\! \dots \!\to\! X_T$ be a Markov chain on a common
(measurable) support, and assume all marginal densities are strictly
positive.  Denote the marginals by $p_t := p(X_t)$ and the one–step kernels by
\[
  P_t(x,\cdot)\;:=\;p\bigl(X_t=\cdot \mid X_{t-1}=x\bigr),
  \qquad t=1,\dots,T.
\]
Then
\[
  D_{\mathrm{KL}}\!\bigl(p_T \,\|\, p_0\bigr)
  \;\le\;
  T \;
  \max_{1\le t\le T}
  \mathbb{E}_{p_{t-1}}
    \Bigl[
      D_{\mathrm{KL}}\!\bigl(
        P_t(X_{t-1},\cdot) \,\|\, p_0
      \bigr)
    \Bigr].
\]
\end{lemma}

\begin{proof}
\textbf{Step 1: Construct an independent reference process.}
Define
\[
  q(X_{0:T}) \;:=\; p_0(X_0)\;\prod_{t=1}^{T} p_0(X_t),
\]
so that each $X_t$ under $q$ is an i.i.d. draw from $p_0$ and
$q(X_T)=p_0$ by construction.

\textbf{Step 2: Chain rule for KL divergence.}
Because $p$ is Markov,
\(
  p(X_t\mid X_{0:t-1}) = P_t(X_{t-1},\cdot).
\)
Applying the chain rule,
\[
  D_{\mathrm{KL}}\!\bigl(p(X_{0:T})\,\|\,q(X_{0:T})\bigr)
  \;=\;
  \sum_{t=1}^{T}
    \mathbb{E}_{p_{t-1}}
      \Bigl[
        D_{\mathrm{KL}}\!\bigl(
          P_t(X_{t-1},\cdot)\,\|\,p_0
        \bigr)
      \Bigr].
\]

\textbf{Step 3: Data–processing inequality.}
For the measurable map $f(X_{0:T})=X_T$,
\[
  D_{\mathrm{KL}}\!\bigl(p_T\,\|\,p_0\bigr)
  \;=\;
  D_{\mathrm{KL}}\!\bigl(p\!\circ\!f^{-1}\,\|\,q\!\circ\!f^{-1}\bigr)
  \;\le\;
  D_{\mathrm{KL}}\!\bigl(p(X_{0:T})\,\|\,q(X_{0:T})\bigr).
\]

\textbf{Step 4: Combine the bounds.}
Putting the previous two displays together,
\[
  D_{\mathrm{KL}}\!\bigl(p_T\,\|\,p_0\bigr)
  \;\le\;
  \sum_{t=1}^{T}
    \mathbb{E}_{p_{t-1}}
      \Bigl[
        D_{\mathrm{KL}}\!\bigl(
          P_t(X_{t-1},\cdot)\,\|\,p_0
        \bigr)
      \Bigr].
\]
Let
\(
  c_t :=
  \mathbb{E}_{p_{t-1}}
    \bigl[
      D_{\mathrm{KL}}\bigl(P_t(X_{t-1},\cdot)\,\|\,p_0\bigr)
    \bigr]
\)
and $c_{\max}:=\max_{1\le t\le T} c_t$.
Then $\sum_{t=1}^{T} c_t \le T\,c_{\max}$, yielding the claimed inequality.
\end{proof}
\begin{theorem}[KL bound across $K$ GA generations]\label{thm:divergence_k_generation}
Under the elitism update rule
\[
  p(\boldsymbol{\tau}_k\mid\boldsymbol{\tau}_{k-1})
  \;=\;
  \rho\,p^{\mathrm{off}}(\boldsymbol{\tau}_k\mid\boldsymbol{\tau}_{k-1})
  +(1-\rho)\,p_{k-1},
  \qquad 0\le\rho\le1,
\]
we define \( p_k \) as the probability distribution of the \( k \)-th generation solutions, which is equivalent to the probability distribution \( p(\boldsymbol{\tau}_k) \) for any trajectory \( \boldsymbol{\tau}_k \) sampled from this generation. The following holds:
\[
  D_\mathrm{KL}(p_K\|p_0)
  \;\le\;
  \rho\,K\,
  \max_{1\le k\le K}
  \mathbb{E}_{\boldsymbol{\tau}_{k-1}\sim p_{k-1}}
  \Bigl[
    D_\mathrm{KL}\bigl(
      p^{\mathrm{off}}(\boldsymbol{\tau}_k\mid\boldsymbol{\tau}_{k-1})
      \,\|\,p_{k-1}
    \bigr)
  \Bigr].
\]
\end{theorem}

\begin{proof}
\textbf{Step 1: Convexity of KL divergence.}  
\[
  D_\mathrm{KL}\!\bigl(
    p(\boldsymbol{\tau}_k\mid\boldsymbol{\tau}_{k-1})
    \,\|\,p_{k-1}
  \bigr)
  \;\le\;
  \rho\,
  D_\mathrm{KL}\!\bigl(
    p^{\mathrm{off}}(\boldsymbol{\tau}_k\mid\boldsymbol{\tau}_{k-1})
    \,\|\,p_{k-1}
  \bigr).
\]

\textbf{Step~2: Apply Lemma~\ref{lem:markov_chains_p0}.}  
Viewing $\boldsymbol{\tau}_0\!\to\!\dots\!\to\!\boldsymbol{\tau}_K$ as a Markov chain with the above kernel and substituting the bound from Step 1,
\[
  D_\mathrm{KL}(p_K\|p_0)
  \;\le\;
  \rho\,K\,
  \max_{k}
  \mathbb{E}_{\boldsymbol{\tau}_{k-1}\sim p_{k-1}}
  \Bigl[
    D_\mathrm{KL}\bigl(
      p^{\mathrm{off}}(\boldsymbol{\tau}_k\mid\boldsymbol{\tau}_{k-1})
      \,\|\,p_{k-1}
    \bigr)
  \Bigr].
\]
\end{proof}
\begin{corollary}[KL between offspring and parents]\label{thm:divergence_offspring_parents}
Let the offspring distribution be the convex mixture
\[
  p^{\mathrm{off}}(\boldsymbol{\tau}_k\mid\boldsymbol{\tau}_{k-1})
  \;=\;
  \alpha\,q^{\mathrm{cross}}(\boldsymbol{\tau}_k)
  +\beta\,q^{\mathrm{mutate}}(\boldsymbol{\tau}_k)
  +\gamma\,q^{\mathrm{elite}}(\boldsymbol{\tau}_k),
  \quad \alpha+\beta+\gamma=1.
\]
If the elite component copies parents verbatim,
\(q^{\mathrm{elite}}=p_{k-1}\),
then
\[
  D_\mathrm{KL}\!\bigl(
    p^{\mathrm{off}}(\boldsymbol{\tau}_k\mid\boldsymbol{\tau}_{k-1})
    \,\|\,p_{k-1}
  \bigr)
  \;\le\;
  \alpha\,D_\mathrm{KL}\!\bigl(
    q^{\mathrm{cross}}\,\|\,p_{k-1}
  \bigr)
  +\beta\,D_\mathrm{KL}\!\bigl(
    q^{\mathrm{mutate}}\,\|\,p_{k-1}
  \bigr).
\]
\end{corollary}

\begin{proof}
By convexity of KL divergence, we have
\[
  D_\mathrm{KL}\bigl(
    p^{\mathrm{off}} \,\|\, p_{k-1}
  \bigr)
  \le
  \alpha\,D_\mathrm{KL}(q^{\mathrm{cross}}\,\|\,p_{k-1})
  +\beta\,D_\mathrm{KL}(q^{\mathrm{mutate}}\,\|\,p_{k-1})
  +\gamma\,D_\mathrm{KL}(q^{\mathrm{elite}}\,\|\,p_{k-1}).
\]
Since \(q^{\mathrm{elite}} = p_{k-1}\), the last term vanishes, yielding the stated bound.
\end{proof}

\begin{lemma}[Crossover as a measurable transformation]
    \label{lem:crossover_measurable}
    Let $f$ be a fixed retained fragment selected by the OX crossover operator. Then the transformation from a parental fragment $r$ to the offspring fragment induced by OX defines a deterministic and support-preserving mapping $g_f$ such that
    \[
        q(r \mid f) = p(g_f^{-1}(r) \mid f).
    \]
    In particular, this mapping $g_f$ can be viewed as a measurable transformation over the conditional fragment space $\mathcal R_f$.
\end{lemma}
\begin{proof}
    Our crossover operator selects fixed parent individuals for pairing rather than randomly. Without loss of generality, we focus on parent $p_2$ and offspring $o_1$.

    Given $f = p_1[s:e]$, the OX operator deterministically fills the remaining positions of $o_1$ using nodes from $p_2$, preserving their relative order. That is, the perturbed fragment $r$ of the offspring $o_1$ is constructed by removing the nodes in $f$ from $p_2$ and placing the remaining nodes into the empty positions of $o_1$ in a fixed cyclic order.

    This process defines a transformation from the parental fragment $\text{segment}(p_2, r)$ to the offspring fragment $\text{segment}(o_1, r)$, conditional on $f$. As the transformation preserves the relative order of elements in $r$, it is deterministic and support-preserving.

    Therefore, for a fixed retained fragment $f$, the OX operator can be viewed as inducing a measurable transformation $g_f$ over the space of fragments $r$, such that
    \[
    \text{segment}(o_1, r) = g_f(\text{segment}(p_2, r)),
    \]
    and correspondingly, the induced distribution over offspring fragments satisfies
    \[
    q(r \mid f) = p(g_f^{-1}(r) \mid f).
    \]
    The same reasoning applies symmetrically to $o_2$ and $p_1$.
\end{proof}

\begin{theorem}[KL bound for the crossover and mutation operators]
\label{thm:divergence_crossover}
Let each trajectory $\boldsymbol{\tau}$ be decomposed as $[f, r]$, where $f$ is the retained fragment
and $r$ the fragment perturbed via genetic operators (e.g., filled from a partner parent in crossover, or modified in mutation).
If the genetic operator leaves $f$ unchanged and only permutes or perturbs $r$, then
\[
  D_\mathrm{KL}\!\bigl(q(\boldsymbol{\tau})\,\|\,p(\boldsymbol{\tau})\bigr)
  \;\le\;
  \mathbb{E}_{f\sim p}\!
  \Bigl[
    \log\!\frac{\max p(r\mid f)}{\min p(r\mid f)}
  \Bigr].
\]
\end{theorem}

\begin{proof}
\textbf{Step~1: Split by fragments.}  
Since $\boldsymbol{\tau} = [f, r]$ is a unique decomposition,
\[
  D_\mathrm{KL}(q\|p) = \sum_{f,r} q(f, r) \log \frac{q(f, r)}{p(f, r)}.
\]

\textbf{Step~2: Marginal equality on $f$.}  
By construction, the operator (crossover or mutation) preserves $f$, hence $q(f) = p(f)$, so
\[
  D_\mathrm{KL}(q\|p)
  = \mathbb{E}_{f\sim p}
  \Bigl[
    \sum_r q(r\mid f) \log \frac{q(r\mid f)}{p(r\mid f)}
  \Bigr].
\]

\textbf{Step~3: Bounded perturbation on $r$.}  
Whether $q(r \mid f)$ arises via permutation (crossover) or local modification (mutation), it satisfies:
\begin{itemize}
  \item $\sum_r q(r \mid f) = 1$
  \item $q(r \mid f)$ and $p(r \mid f)$ share support
  \item $q(r \mid f)$ is a rearrangement of $p(r \mid f)$
\end{itemize}
Thus, define a measure-preserving transformation $g_f$ so that
$q(r \mid f) = p(g_f^{-1}(r) \mid f)$. Then:
\[
  \sum_{r} q(r\mid f) \log \frac{q(r\mid f)}{p(r\mid f)}
  =
  \sum_r p(r \mid f) \log \frac{p(r \mid f)}{p(g_f(r) \mid f)}
  \le
  \log \frac{M_f}{m_f},
\]
where $M_f = \max p(r \mid f)$ and $m_f = \min p(r \mid f)$.

\textbf{Step~4: Take expectation over $f$.}  
Substitute this bound into Step~2:
\[
  D_\mathrm{KL}(q\|p) \le \mathbb{E}_{f \sim p} \left[ \log \frac{M_f}{m_f} \right],
\]
completing the proof.
\end{proof}

\begin{corollary}[Convex mixture of deterministic transforms]
\label{cor:convex_pushforward}
If, given the retained fragment \(f\), a genetic operator
acts as a convex combination of finitely many
support‑preserving deterministic maps \(g_{f,j}\) (e.g. \ node-replacement mutation),
i.e.\ \(q(r\mid f)=\sum_j w_j p(g_{f,j}^{-1}(r)\mid f)\),
then the bound in Theorem~\ref{thm:divergence_crossover} still holds.
\end{corollary}

\begin{proof}
KL is convex in its first argument, so
\(D_{\mathrm{KL}}(\sum_j w_j q_j\|p)
      \le \sum_j w_j D_{\mathrm{KL}}(q_j\|p)
      \le \log(M_f/m_f)\).
Take expectation over \(f\).
\end{proof}

\begin{theorem}[Overall KL bound]
    Summing up aforementioned Thorems, we have:
    \begin{align*}
        & D_{\text{KL}}(p(\boldsymbol{\tau}_K) \| p(\boldsymbol{\tau}_0))\le \\&\rho K \Bigg(\alpha \mathbb{E}_{p(f_{\text{cross}})}\bigg[\log \frac{\max{p(r_{\text{cross}}|f_{\text{cross}})}}{\min p(r_{\text{cross}}|f_{\text{cross}})}\bigg] + \beta \mathbb{E}_{p(f_{\text{mutate}})}\bigg[\log\frac{\max p(r_{\text{mutate}}|f_{\text{mutate}})}{\min p(r_{\text{mutate}}|f_{\text{mutate}})}\bigg]\Bigg)
    \end{align*}
\end{theorem}

\begin{theorem}[Policy Gradient Difference Upper Bound]\label{thm:gradient_difference_short}
If every trajectory gradient is L2‑clipped to~1, i.e.\ If $\|\nabla J(\boldsymbol{\tau})\|_2\le1$ almost surely, then
\[
  \Bigl\|
    \mathbb{E}_{\boldsymbol{\tau}_K\sim p_K}[\nabla J(\boldsymbol{\tau}_K)]
    -\mathbb{E}_{\boldsymbol{\tau}_0\sim p_0}[\nabla J(\boldsymbol{\boldsymbol{\tau}}_0)]
  \Bigr\|_2
  \;\le\;
  \sqrt{2\,D_\mathrm{KL}\!\bigl(p_K\|p_0\bigr)}.
\]
\end{theorem}

\begin{proof}
\textbf{Step~1: Expectation--TV inequality.}  
Apply lemma~\ref{lem:exp_tv} with $f=\nabla J$, $P=p_0$, $Q=p_K$ and
$\|f\|_{2,\infty}\le1$ (In our implementation, we apply L2 gradient clipping) gives
\[
  \Bigl\|\mathbb{E}_{p_K}[f]-\mathbb{E}_{p_0}[f]\Bigr\|_2
  \le2\,D_\mathrm{TV}(p_K\|p_0).
\]

\textbf{Step~2: Convert TV to KL.}  
Pinsker’s inequality yields
$D_\mathrm{TV}(p_K\|p_0)\le\sqrt{\tfrac{1}{2}\,D_\mathrm{KL}(p_K\|p_0)}$.
Substituting into Step~1 produces the claimed bound.
\end{proof}

\section{Detailed Experiment Settings}
\label{app:detailed_exp_settings}
This section presents the complete experimental setup used throughout the paper, including baseline hyperparameters, inference strategies, and hardware resources.
\subsection{Baseline Hyperparameters}

\begin{table}[H]
  \centering
  \caption{\textbf{Hyperparameters used for AM-based models~\citep{kool2018attention} across COP tasks.} 
Learning rate, encoder layers, embedding dimension, attention heads and feedforward dimension are consistent across all tasks. PCTSP and OP additionally share training schedule parameters (epochs, epoch size, total steps). Settings are shared unless specified as A / B for 50-node / 100-node instances.}
  \label{tab:am_hyper_parameters}
  \begin{tabular}{lcccc}
    \toprule
    \textbf{Hyperparameter} & \textbf{TSP50/100} & \textbf{CVRP50/100} & \textbf{PCTSP100} & \textbf{OP100} \\
    \midrule
    Learning Rate                & \multicolumn{4}{c}{1e-4} \\
    Encoder Layers               & \multicolumn{4}{c}{6} \\
    Embedding Dimension          & \multicolumn{4}{c}{128} \\
    Attention Heads              & \multicolumn{4}{c}{8} \\
    Feedforward Dimension        & \multicolumn{4}{c}{512} \\
    Batch Size                   & \multicolumn{2}{c}{64} & \multicolumn{2}{c}{512} \\
    Epochs                       & 1,000 / 1,500 & 500 / 700 & \multicolumn{2}{c}{100} \\
    Epoch Size (instances)       & 160,000 & 160,000 & \multicolumn{2}{c}{1,280,000} \\
    Total Steps                  & 3.75M & 1.75M & \multicolumn{2}{c}{250K} \\
    \bottomrule
  \end{tabular}
\end{table}

\begin{table}[H]
  \centering
  \caption{\textbf{Hyperparameters used for POMO~\citep{kwon2020pomo} across different COP tasks.} Settings are shared unless specified as A / B for 50-node / 100-node instances.}
  \label{tab:pomo_hyperparameters}
  \begin{tabular}{lcc}
    \toprule
    \textbf{Hyperparameter} & \textbf{TSP50/100} & \textbf{CVRP50/100} \\
    \midrule
    Learning Rate                & \multicolumn{2}{c}{1e-4} \\
    Weight Decay                & \multicolumn{2}{c}{1e-6} \\
    Encoder Layers              & \multicolumn{2}{c}{6} \\
    Embedding Dimension         & \multicolumn{2}{c}{128} \\
    Attention Heads             & \multicolumn{2}{c}{8} \\
    Feedforward Dimension       & \multicolumn{2}{c}{512} \\
    Batch Size                  & \multicolumn{2}{c}{64} \\
    Epochs                      & 1,000 / 1,500 & 500 / 700 \\
    Epoch Size (instances)      & \multicolumn{2}{c}{160,000} \\
    Total Steps                 & 3.75M & 1.75M \\
    \bottomrule
  \end{tabular}
\end{table}

\begin{table}[H]
  \centering
  \caption{\textbf{Hyperparameters used for Sym-NCO~\citep{kim2022sym} on TSP and CVRP tasks.}
  Settings are shared unless noted as A / B for 50-node / 100-node instances.}
  \label{tab:symnco_hyperparameters}
  \begin{tabular}{lcc}
    \toprule
    \textbf{Hyperparameter} & \textbf{TSP50/100} & \textbf{CVRP50/100} \\
    \midrule
    Learning Rate                        & \multicolumn{2}{c}{1e-4} \\
    Weight Decay                        & \multicolumn{2}{c}{1e-6} \\
    Encoder Layers                      & \multicolumn{2}{c}{6} \\
    Embedding Dimension                 & \multicolumn{2}{c}{128} \\
    Attention Heads                     & \multicolumn{2}{c}{8} \\
    Feedforward Dimension               & \multicolumn{2}{c}{512} \\
    Batch Size                          & \multicolumn{2}{c}{64} \\
    Epochs                              & 1,000 / 1,500 & 500 / 700 \\
    Epoch Size (instances)              & \multicolumn{2}{c}{160,000} \\
    Total Steps                         & 3.75M & 1.75M \\
    \midrule
    $\alpha$ (weight for $\mathcal{L}_{\text{inv}}$)      & 0.1 & 0.2 \\
    $\beta$ (weight for $\mathcal{L}_{\text{ss}}$)        & 1.0 & 1.0 \\
    $N$ (solutions per problem)        & 50 / 100 & 50 / 100 \\
    $L$ (rotation count per instance) & 2 & 2 \\
    \bottomrule
  \end{tabular}
\end{table}

Table~\ref{tab:am_hyper_parameters}, Table~\ref{tab:pomo_hyperparameters} and Table~\ref{tab:symnco_hyperparameters} lists the training hyperparameters used for baseline models on various combinatorial optimization problems. Parameters include learning rate, number of layers, embedding dimensions, batch size, number of epochs, and total training steps, with separate settings for different instance sizes when applicable.

\subsection{Multi-start and Data Augmentation Decoding Strategy}

To improve inference quality, we apply a multi-start decoding strategy by running the model from each possible initial node and selecting the best trajectory~\citep{kwon2020pomo}. Additionally, we perform geometric data augmentation by rotating each instance by $0^\circ$, $90^\circ$, $180^\circ$, and $270^\circ$, with and without horizontal reflection, resulting in 8 variants~\citep{kwon2020pomo}. The model decodes each variant independently, and the best solution is retained, as shown in Table~\ref{tab:tsp_cvrp_results}.

\subsection{Computing Resources}
All experiments and evaluations were conducted on a system equipped with two \textit{AMD EPYC 7742 64-Core Processors CPU} (128 physical cores, 256 threads in total) and an \textit{NVIDIA RTX 3090 GPU}.

\section{Implementation Details of Proposed Method}
\label{app:implement_details}
This section provides detailed information on the implementation of the proposed Evolution Augmentation Mechanism (EAM), including algorithmic structure and hyperparameters configurations.

\subsection{Pseudo Code}
We provide the pseudo code for the proposed Evolution Augmentation Mechanism (EAM), as shown in Algorithm~\ref{alg:eam_framework}. The procedure includes policy initialization, environment sampling, optional application of the genetic algorithm based on pre-defined conditions, and subsequent policy updates using REINFORCE. The red-highlighted areas represent the core part of EAM and are the major differences from the standard RL paradigm, demonstrating how our method extends the current DRL-NCO framework's mainstream training paradigm as a plug-and-play module.

\newcommand{\GeneticAlgorithm}{GeneticAlgorithm}
\newcommand{\REINFORCE}{ApplyREINFORCE}

\begin{algorithm}[H]
\caption{Evolution Augmentation Mechanism}
\label{alg:eam_framework}
\begin{algorithmic}

\Require Base nerual solver $p_\theta$, EA Probs $\epsilon$, EA Epochs $\kappa$, num evolution generations $K$, selection rate $\rho$, crossover rate $\alpha$, mutation rate $\beta$.

\Procedure{Train}{$p_\theta, \epsilon, \kappa, K, \rho, \alpha, \beta$}

    \State Initialize $p_\theta$
    \State $epoch \gets 0$
    \While{not converged}
        \State Randomly initialize environment $E$
        \State $\mathrm{P}_0 \gets p_\theta(E)$ 
        \If{\textcolor{red}{use genetic algorithm}} \Comment{Conditioned on $\epsilon, \kappa$}
            \State \textcolor{red}{$\mathrm{P}_K \gets $ \CommonCall{\GeneticAlgorithm}{$\mathrm{P}_0,K,\rho,\alpha,\beta$}}
            \State \textcolor{red}{\CommonCall{\REINFORCE}{$\mathrm{P}_K, p_\theta$}}
        \EndIf

        \State \CommonCall{\REINFORCE}{$\mathrm{P}_0, p_\theta$}
        
        \State $epoch \gets epoch + 1$
    \EndWhile

    \State \Return $p_\theta$
\EndProcedure

\Procedure{Test}{$p_\theta$}

    \State Randomly initialize environment $E$
    \State $\mathrm{P}_0 \gets p_\theta(E)$

    \State \Return Cost($\mathrm{P}_0$)
\EndProcedure
\end{algorithmic}
\end{algorithm}

\subsection{Task-aware evolutionary hyperparameter selection}

\begin{table}[htbp]
  \centering
  \caption{\textbf{Task-aware evolutionary hyperparameters for each COP.} $N$ is the number of evolution generations, $\rho$ is the selection rate, $\alpha$ is the crossover rate, $\beta$ is the mutation rate, 
  ``EA Epochs'' indicates the number of training epochs using evolution augmentation, and ``EA Probs'' is the probability of applying evolution at each step. Settings are shared unless noted as A / B for 50-node / 100-node instances.}
  \label{tab:evolutionary_hyperparameters}
  \begin{tabular}{lcccc}
    \toprule
    \textbf{Parameter} & \textbf{TSP50/100} & \textbf{CVRP50/100} & \textbf{PCTSP100} & \textbf{OP100} \\
    \midrule
    $K$ (evolution generations)      & 5         & 3         & 5   & 2   \\
    $\rho$ (selection rate)          & 0.2       & 0.2       & 0.2 & 0.4 \\
    $\alpha$ (crossover rate)        & 0.6       & 0.6       & 0.6 & 0.0 \\
    $\beta$ (mutation rate)          & 0.05      & 0.10      & 0.05 & 0.5 \\
    EA Epochs                        & 500 / 700 & 200 / 300 & 20  & 50  \\
    EA Probs                         & \multicolumn{4}{c}{0.01} \\
    \bottomrule
  \end{tabular}
\end{table}

As shown in Table~\ref{tab:evolutionary_hyperparameters}, we adopt task-aware evolutionary hyperparameters tailored to each COP. For the OP task, we choose a mutation-dominant configuration $(\alpha = 0.0, \beta = 0.5)$, based on both implementation details and theoretical considerations. Specifically, we conduct an ablation study comparing this setting with a crossover-dominant alternative $(\alpha = 0.6, \beta = 0.05)$, and observe that the mutation-dominant configuration leads to faster convergence and more stable training dynamics.

From an implementation perspective, the mutation operator in OP replaces only a single node per operation and must ensure that the resulting tour remains feasible under strict length constraints. This design aligns with the OP objective: since the total reward depends primarily on which nodes are selected rather than their ordering, mutation is more effective at introducing high-reward nodes with minimal disruption to tour structure. In contrast, crossover typically only reorders already-visited nodes, offering less opportunity for reward improvement. Compared to crossover, our mutation strategy provides a more efficient mechanism for guiding policy learning.

Futhermore, these implementation-level characteristics are suggestive of differences in KL divergence between successive policies. As shown in Theorem~\ref{thm:main}, the divergence depends on the ratio $\log \frac{\max p(r\mid f)}{\min p(r\mid f)}$, where $f$ is the retained fragment and $r$ the replaced part of a solution. Mutation operators in OP only modify a single node, and the replacement must satisfy the strict tour-length constraint. This limits the feasible support of $p(r \mid f)$, reducing the variability of the conditional distribution and tightening the KL bound. In contrast, OX crossover replaces longer subpaths, increasing the number of feasible completions and hence the $\max/\min$ ratio. Thus, crossover tends to introduce higher structural divergence between generations. Based on these two considerations—the implementation-level alignment with task structure and the theoretically tighter KL bound—we favor the mutation-dominant configuration over crossover-dominant alternatives.

\end{document}